\documentclass{article}
 
\usepackage{microtype}
\usepackage{graphicx}
\usepackage{subfigure}
\usepackage{subcaption}
\usepackage{booktabs} % for professional tables
\usepackage{color}
\usepackage[table]{xcolor}
\usepackage{float}
\usepackage{enumitem}
\usepackage{natbib}

\usepackage{algorithm}
\usepackage{algorithmic}
\usepackage{amsmath}
\usepackage{amssymb}
\usepackage{mathtools}
\usepackage{amsthm}

\usepackage{hyperref}
\usepackage{multirow}

\usepackage[preprint]{icml2026}

\usepackage[capitalize,noabbrev]{cleveref}

\theoremstyle{plain}
\newtheorem{theorem}{Theorem}%[section]

\newtheorem{lemma}[theorem]{Lemma}

\theoremstyle{definition}

\theoremstyle{remark}

\usepackage[textsize=tiny]{todonotes}

\icmltitlerunning{Mano: Restriking Manifold Optimization for LLM Training}

\begin{document}

\twocolumn[
  \icmltitle{Mano: Restriking Manifold Optimization for LLM Training}
  \icmlsetsymbol{equal}{*}

  \begin{icmlauthorlist}
    \icmlauthor{Yufei Gu}{hkustgz}
    \icmlauthor{Zeke Xie$\dagger$}{hkustgz}
  \end{icmlauthorlist}

  \icmlaffiliation{hkustgz}{xLeaF Lab, The Hong Kong University of Science and Technology (Guangzhou)}

  \icmlcorrespondingauthor{Yufei Gu}{ygu167@connect.hkust-gz.edu.cn}
  \icmlcorrespondingauthor{Zeke Xie$\dagger$}{zekexie@hkust-gz.edu.cn}

  \icmlkeywords{Large Language Models, Optimization}
  \vskip 0.3in
]

\printAffiliationsAndNotice{}  % no special notice (required even if empty)

\begin{abstract}
    While large language models (LLMs) have emerged as a significant advancement in artificial intelligence, the hardware and computational costs for training LLMs are also significantly burdensome. Among the state-of-the-art optimizers, AdamW relies on diagonal curvature estimates and ignores structural properties, while Muon applies global spectral normalization at the expense of losing curvature information. In this study, we restriked manifold optimization methods for training LLMs, which may address both optimizers' limitations, while conventional manifold optimization methods have been largely overlooked due to the poor performance in large-scale model optimization. By innovatively projecting the momentum onto the tangent space of model parameters and constraining it on a rotational Oblique manifold, we propose a novel, powerful, and efficient optimizer \textbf{Mano} that is the first to bridge the performance gap between manifold optimization and modern optimizers. Extensive experiments on the LLaMA and Qwen3 models demonstrate that Mano consistently and significantly outperforms AdamW and Muon even with less memory consumption and computational complexity, respectively, suggesting an expanded Pareto frontier in terms of space and time efficiency.
\end{abstract}

\section{Introduction}

Large Language Models (LLMs) are represented as a major milestone in artificial intelligence and computing technology \citep{kumar2024large, gao2025llm}. Nevertheless, training LLMs with billions of parameters requires specialized hardware accelerators, such as GPUs and NPUs, which consumes sunstantial energy and result in enormous computational costs \citep{samsi2023words, chowdhury2025hidden}. Consequently, advanced optimization methods are critical for reducing the training costs of LLMs. 

Adam is widely regarded as one of the most popular optimizers in deep learning, with its variants being extensively used to train deep neural networks (DNNs) across a wide range of tasks \citep{kingma2014adam}. In particular, AdamW, which decouples weight decay from the original gradient update of Adam, has been the dominant optimizer in training LLMs with broad infrastructure support for large-scale distributed training \citep{loshchilov2017decoupled, kunstner2024heavy}. However, by relying on the diagonal estimates of per-parameter curvature, Adam-based optimizers ignore spectral information and fail to capture the subspace structures of the parameter matrix. To address this limitation, Muon introduces the Newton-Schulz iteration to approximate matrix orthogonalization, allowing uniform exploration across all spectral directions in the loss landscape \citep{jordan2024muon}. However, this spectral normalization strategy eliminates the curvature information encoded in the gradient and momentum, which may lead to sub-optimal performance on certain occasions \citep{su2025isotropic}. 

For optimization problems with a so-called manifold structure, algorithms have been developed to exploit this special structure from the perspective of differential geometry and numerical analysis, referred to as manifold optimization methods \citep{absil2008optimization, hu2020brief, fei2025survey}. However, these traditional manifold optimization methods have been overlooked for LLMs. 
In this study, we seek to adapt and extend manifold optimization methods to LLM training, where the underlying manifold structure of natural language or optimal parameters is not explicitly known or well understood, relaxing the constraints and assumptions typically required.
Instead of retracting and constraining the parameters onto some manifold's surfaces, we demonstrate that by only projecting the momentum onto the tangent vector spaces of the parameters and mapping it to the Oblique manifold, the constrained gradient steps emerged as a simple yet effective learning direction in the solution space with strong escape-from-local-minima properties. Furthermore, we demonstrate the empirical advantages of periodically rotating the Oblique manifold, a process equivalent to an alternating column-row normalization scheme applied at each optimization step.

\begin{figure}[ht]
    \centering
    \subfigure[LLaMA-350M / \texttt{Pile}]{\includegraphics[width=0.48\linewidth]{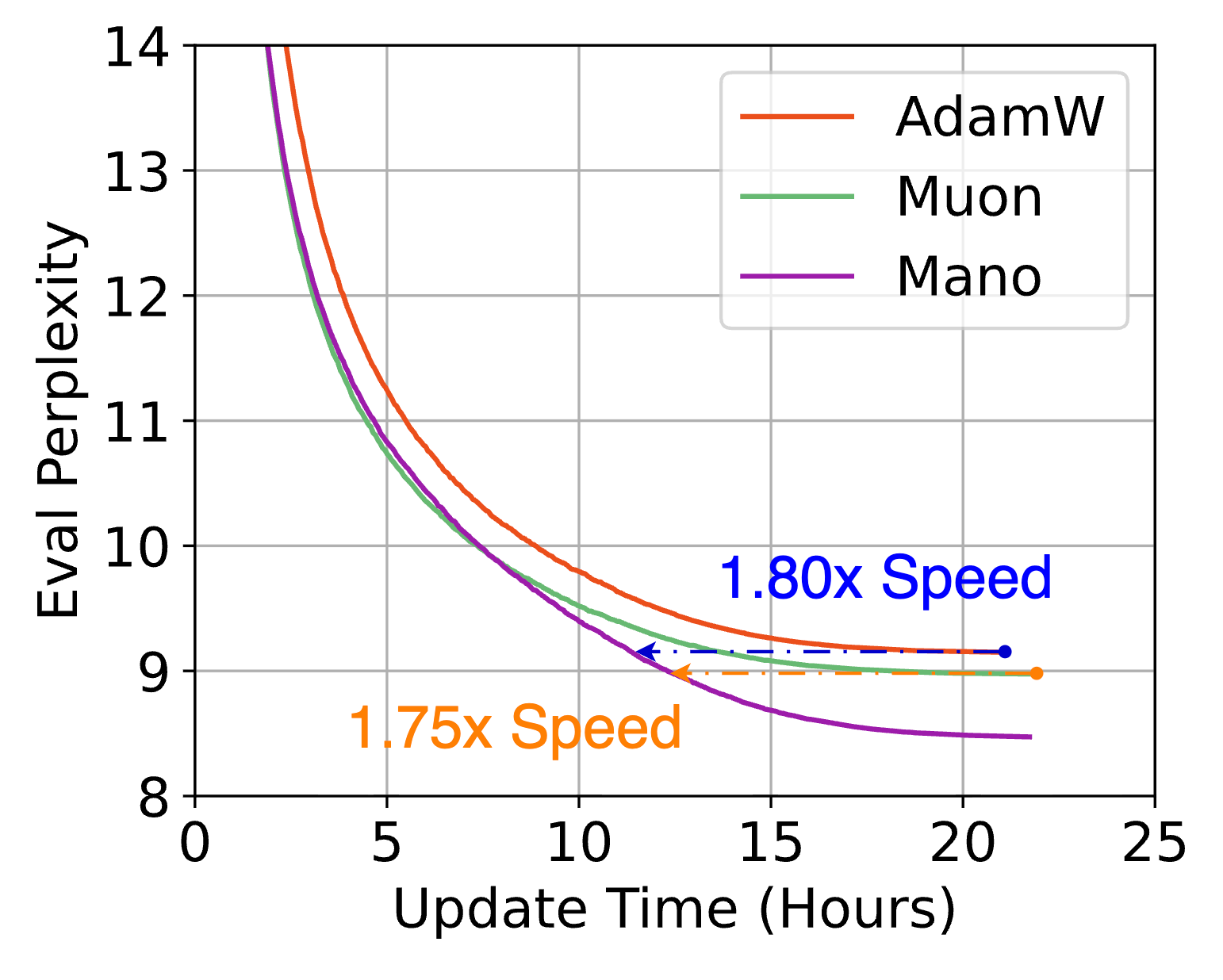}}
    \subfigure[LLaMA-1.3B / \texttt{Pile}]{\includegraphics[width=0.48\linewidth]{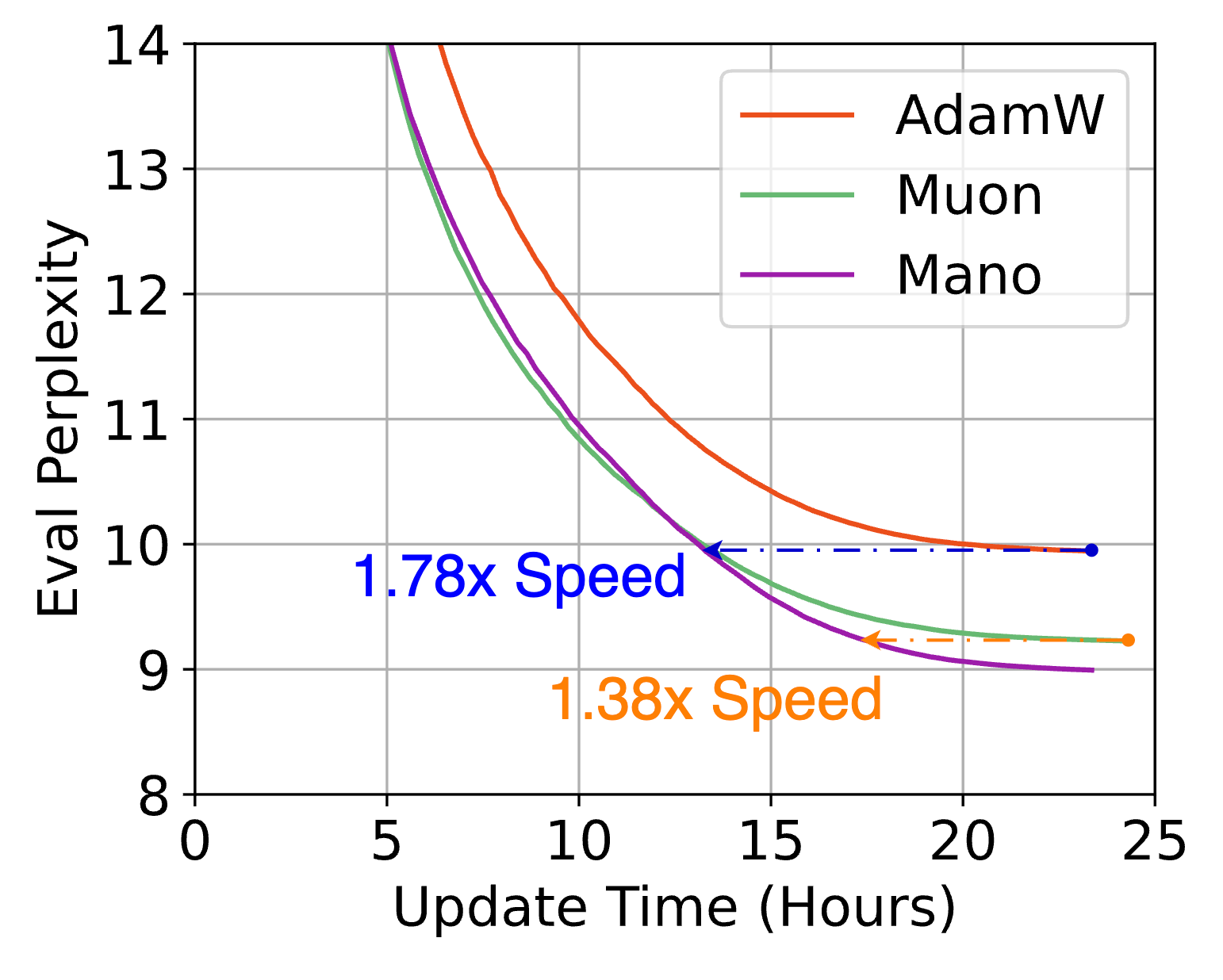}}
    \caption{One day pretraining experiment of LLaMA-350M and -1.3B models on the \texttt{Pile} dataset for $10$B and $2.8$B tokens respectively. Our proposed optimizer Mano achieves $1.75\times$ and $1.38\times$ the convergence speed of Muon in terms of wall-clock time. This advantage is expected to further increase with reduced computational overhead and faster convergence speed.}
    \label{fig:LLaMA-wall-clock-perplexity}
\end{figure}

By restriking manifold optimization with multiple reformed strategies, we propose a new class of optimizer, \textbf{MA}nifold-\textbf{N}ormalized \textbf{O}ptimizer (\textbf{Mano}). Throughout this paper, we seek to unravel the hidden potential of manifold optimization methods in today's infrastructure of training neural networks. This work mainly made three contributions. 

%, which can expand the existing Pareto frontier of optimizers in memory consumption, computational overhead, and sample efficiency. We validate Mano on two model architectures, LLaMA and Qwen3, with different scales and two pretraining corpora to report model- and data-transferability and enforce a more complete comparison to the baseline optimizers AdamW and Muon. We provide further empirical results on learning dynamics and various ablation studies to discuss the critical design choices of Mano. \xie{This part is redundant given the following contributions}

\begin{itemize}[topsep=0pt, itemsep=4pt, partopsep=0pt, parsep=0pt]
    \item To the best of our knowledge, we are the first to revisit and restrike manifold optimization techniques for LLM training and highlight its promising potential when combining with the proposed reform strategies, while traditional manifold methods have been largely overlooked due to poor performance on LLMs.
    \item We design a novel, powerful, and efficient optimizer \textbf{Mano} for LLM training. It consumes less memory and has significantly lower computational complexity than popular modern optimizers, such as Adam and Muon.
    \item Mano is the first manifold optimizer that works well in LLM training, significantly outperforming Adam and Muon in test perplexity across both token consumption and wall-clock time. 
    We also observe that it can effectively update model parameters with reduced gradient variance, theoretically suggesting better convergence. 
    Mano restrikes a promising future of the reformed manifold optimization paradigm for LLM training.
\end{itemize}

\section{Related Works}

This section reviews prior works on efficient optimizers designed for pretraining LLMs and manifold optimization techniques in the field of deep learning.

\subsection{Optimizers for LLM Pretraining}

Adam-based optimizers remain the most widely used optimizers in the field of deep learning, including both the pretraining and fine-tuning of decoder-only transformers \citep{zhao2024deconstructing}. Its popularity stems from its simplified design, flexibility in per-parameter adaptive learning rate, and robustness in performance across diverse domains. However, the first- and second-moment estimates of AdamW consume double the memory footprint of model weights or gradients, resulting in a significant memory overhead, especially for LLMs at scale. Several approaches have emerged to design more memory-efficient optimizers. Adam-mini leverages block-wise learning rate schedules based on Hessian partitions \citep{zhang2024adam, wang2025sharpness}; Lion uses momentum-sign updates to eliminate the need of second moment \citep{chen2023symbolic}, and Cautious-Adam/Lion applies gradient-aligned selective updates \citep{liang2024cautious}; SOAP instead applies AdamW updates in the Shampoo eigenbasis while amortizing eigendecomposition costs across multiple steps \citep{gupta2018shampoo, vyas2024soap}. Other notable methods include SWAN \citep{ma2024swan}, MARS \citep{yuan2024mars}, Sophia \citep{liu2023sophia}, and etc.

Another particularly promising line of research focuses on matrix-based spectral preconditioning methods. Muon, introduced by \citet{jordan2024muon} in 2023, utilized the Newton-Schulz iteration to perform spectral normalization on the update steps. This approximation to the matrix-sign function produces a semi-orthogonal momentum update that normalized the magnitude at all spectral directions, including those low-magnitude but important directions for model generalization. Empirical studies have later extended Muon to scaled-up LLM training with stability, and demonstrated improved efficiency with halved memory consumption in comparison to AdamW \citep{liu2025muon, shah2025practical, team2025kimi, zeng2025glm}. Benchmarking studies also demonstrate that matrix-based optimizers with spectral preconditioning (e.g., Kron, Muon, SOAP) often outperform scalar-based counterparts (e.g., AdamW, Lion, MARS), though no optimizer significantly outperforms in every tested scenario \citep{schlotthauer2025pre, wen2025fantastic, semenov2025benchmarking}. We are reminded that empirical studies cannot investigate all possible regimes, even with theoretically guided insights and standardized experimental setups. In this paper, we offer empirical evidence across varied contexts, prioritizing hypothesis testing and investigation over the assertion of definitive conclusions.

\subsection{Manifold Optimization in Deep Learning}

Geometric optimization methods are designed to exploit the intrinsic geometric structures of the target objective function and have emerged as promising solutions to problems in various fields, including deep learning \citep{fei2025survey}. For objective functions defined on a Riemannian manifold with a differentiable structure and a smooth inner product, various manifold optimization techniques are proposed to find optimal solutions on the manifold, including diverse Riemannian optimizers that have been developed as geometrically-aware counterparts of conventional Euclidean methods, such as SGD, SGD-M, RMSProp, Adam, AdaGrad, AMSGrad, etc. \citep{bonnabel2013stochastic, zhang2016riemannian, becigneul2018riemannian, roy2018geometry}. A broad spectrum of architecture-dependent manifold optimization techniques has been proposed for CNNs \citep{ozay2016optimization, wang2020orthogonal}, RNNs \citep{arjovsky2016unitary, wisdom2016full, huang2018orthogonal, jing2019gated}, GNNs \citep{zhu2020graph, liu2021human, de2023latent}, and other deep learning techniques \citep{zhang2018deep, chaudhry2020continual}. 

While manifold optimization methods are well-established in other deep learning domains, they have been overlooked in the LLM literature and practices, leaving a relatively small fraction of studies exploring geometrically aware training strategies for LLMs. Recent literature has explored parameter-efficient low-rank training for LLMs through the lens of Riemannian manifolds \citep{jiangloram, zhang2024retraction, mo2025parameter, park2025riemannian}. Similar strategies have been extended to gradient tracking \citep{rajabi2024optimizing, rajabi2025subtrack++}, representation regularization \citep{zhang2025multi, kingswell2025sequential, wren2025contextual}, and parameter pruning \citep{liu2024pruning}. %, chenrotpruner}.
Explicit discussions on Mano's relationship to other prior optimizers are provided in the Appendix ~\ref{apx:other-optimizers}.

\section{Preliminaries}

This section revisits the traditional concept of manifold optimization and Riemannian stochastic gradient descent (SGD), discussing how our reformulated method is motivated.

A Riemannian manifold $\mathcal{M}$ is a smooth geometric space equipped with a metric that defines a smoothly varying inner product on the tangent space of $\mathcal{M}$. Manifold optimization concerns the problem of minimizing a real-valued function $f$ over such a manifold  $\mathcal{M}$, i.e.,
\begin{equation}
    \label{eq:manifold-optimization}
    \min_{x \in \mathcal{M}} f(x)
\end{equation}
where $f: \mathcal{M} \rightarrow \mathbb{R}$. With the above definition, vanilla SGD that optimizes a loss function defined over the Euclidean vector space $\mathbb{R}^n$ can be interpreted as operating in a Riemannian manifold $(\mathbb{R}^n, g_{ij})$ with metric $g_{ij} = \delta_{ij}$. \citet{bonnabel2013stochastic} generalizes SGD to perform gradient updates on Riemannian manifolds through the following operations:
\begin{equation}
    \label{eq:rsgd-update-rule}
    \begin{cases}
        g_t = \nabla f(\theta_t) \\
        v_t = \mathbf{proj}_{\mathcal{T}_{\theta_t}\mathcal{M}} (g_t) \\
        \theta_{t+1} = \mathbf{proj}_{\mathcal{M}} (\theta_t - \eta_t v_t) \approx \exp_{\theta_t}(-\eta_t v_t)
    \end{cases}
\end{equation}
First, the gradient vector $g_t$ is orthogonally projected onto the tangent space $\mathcal{T}_{\theta_t}\mathcal{M}$ to determine the steepest ascent direction $v_t$ for the objective function. A gradient step is then performed in the direction of $- \eta_t v_t$ and mapped back to the manifold surface. While the exponential map provides the geometrically exact update, its high computation cost is often replaced with numerical retractions. Retraction typically performs orthogonal projections back onto the manifold $\mathcal{M}$, serving as efficient first-order approximations that maintain the manifold constraint \citep{bonnabel2013stochastic}.

However, traditional Riemannian manifold optimization strategies often fail to generalize to modern neural networks on general tasks, particularly LLMs. Certain manifolds--such as the Stiefel manifold--often require expensive matrix decompositions (e.g., SVD, QR, etc.), which impose inefficiency in optimization. Furthermore, manifold constraints can restrict the model's ability to explore the loss landscape, especially when the geometric structure of the chosen manifold does not align with the underlying objective function. Instead of framing natural language modeling or the optimal parameter solution $\theta^*$ within an arbitrary Riemannian manifold, we hypothesize that the learning trajectory and each constituent update step can be mapped onto some ``smooth surfaces" with geometric structures that facilitate convergence and help to escape from local minima. Motivated by this hypothesis, we will discuss how we designed a new optimizer in the next section.

\section{Methodology}

%In this section, we will present our reformulation of the traditional manifold optimization strategies into a momentum optimizer and the setup of rotating manifold normalization. 
% \xie{Consistently use manifold optimization in this paper, while Riemannian optimization is the same thing.}
In this section, we reformulate traditional manifold optimization strategies into a momentum-driven optimizer and detail the configuration of rotating manifold normalization.

\subsection{Reformed Manifold Optimization} 
%\xie{I think we need to consistently use manifold or Remannian. As our method adopts manifold, I recommend manifold instead of Riemannian.}

To formulate our manifold optimization methodology, we define a tangent space projector as $\mathbf{proj}_{\mathcal{T}_{P}\mathcal{M}} (Q)$ of matrix $Q$ on the first-order approximation of the manifold surface $\mathcal{M}$ around $P$, and a manifold normalization operation $\mathcal{N}_{\mathcal{M}}(A)$ to constrain matrix $A$ on the target manifold. For weight $\theta_t \in \mathbb{R}^{m \times n}$, gradient $g_t$, and learning rate $\eta_t$ at timestep $t$, we arrived at the update rule as follow:
\begin{equation}
    \label{eq:mano-update-rule}
    \begin{cases}
        g_t = \nabla f(\theta_t) \\
        \hat{\theta}_t = \mathcal{N}_{\mathcal{M}} (\theta_t) \\
        v_t = \mathbf{proj}_{\mathcal{T}_{\hat{\theta}_t}\mathcal{M}} (g_t) \\
        \hat{v}_t = \mathcal{N}_{\mathcal{M}} (v_t) \\
        \theta_{t+1} = \theta_t - \eta_t \, \hat{v}_t%, \quad \alpha=0.2\sqrt{m}
    \end{cases}
\end{equation}

We emphasize that this reformed update rule with manifold constraint is different from the original definition of manifold optimization stated in Eq.~\ref{eq:manifold-optimization}, which defined the function $f$ (parameters) on the Riemannian manifold. Our update rule can be viewed as imposing a soft manifold constraint: it projects each update step onto the manifold surface defined by the parameters $\theta$, while keeping the objective and solution unchanged in the Euclidean space.

\subsection{Manifold Selection and Design}

Among popular matrix manifolds, we select the Oblique manifold for our update rules due to its computational efficiency in manifold normalization. This choice is also driven by the hypothesis that smoother surface geometry reduces trajectory distances, thereby aiding convergence. Tab.~\ref{tab:geodesic-distance} presents the average geodesic distance across $1000$ consecutive update steps of a Qwen3-$0.6$B model trained with AdamW. Our observations reveal that the Oblique manifold yields the shortest geodesic distance compared to the Sphere and Stiefel manifolds, providing an intuitive geometric justification for our design, which better captures the model's natural learning trajectory.

\begin{table}[ht]
    \centering
    \caption{Average geodesic distance measured on the Oblique, Sphere, and Stiefel manifold of $1000$ consecutive update steps of Qwen3-$0.6$B trained with the AdamW optimizer. The distance metrics are reported separately between the attention projections (Q, K, V, O) and MLP layers. }
    \newcommand{\vsp}{\rule{0pt}{10pt}}
    \begin{tabular}{|c|c|c|c|}
        \hline
        \vsp Geodesic distance & Oblique & Sphere & Stiefel \\
        \hline
        Attention      & 36.50 & 41.12 & 58.52 \\
        MLP Layer       & 21.13 & 37.82 & 53.48 \\
        \hline
    \end{tabular}
    \label{tab:geodesic-distance}
\end{table}

\textbf{Notations} We denote the Oblique manifold as $\mathcal{OB}(n,m)$, the set of $\mathbb{F}-$valued matrices with unit norm column, endowed with the metric from the embedding \citep{roberts1960random, absil2006joint, huang2020projection}. We define the following operators:
\begin{itemize}[topsep=0pt, itemsep=0pt, partopsep=0pt, parsep=0pt]
    \item Element-wise product ($\odot$): $P \odot Q \triangleq \bigl( P_{ij} Q_{ij} \bigr)_{i,j}$.
    \item Element-wise division ($\oslash$): $P \oslash Q \triangleq \bigl(P_{ij} / Q_{ij} \bigr)_{i,j}$.
    \item Dimension-wise inner product ($\langle \cdot,\cdot \rangle_k$): For $j \in \{0, \ldots, n_k-1$\} and the $k$-th dimension, the $j$-th component $\langle Q, P \rangle_d^{(j)} = \langle Q^{(j)}, P^{(j)} \rangle$.
    \item Dimension-wise norm ($\| \cdot \|_{2,k}$): For the $k$-th dimension, $\|P\|^{(i)}_{2,d} = \|P^{(i)}\|_2$. We further denote $\|P\|_{2,0}^{(i)} = \|P_{i,:}\|_2$ and $\|P\|_{2,1}^{(j)} = \|P_{:,j}\|_2$ for column- and row-wise norm respectively.
    %\xie{I think the notations for dimension-wise operators may still confuse reviewers. We'd better explain this for row/column-wise operations, respectively.}
\end{itemize}
We thus define the orthogonal projection of a vector $Q$ onto the tangent space $\mathcal{T}_{P}\mathcal{OB}$ at point $P$ as
\begin{equation}
    \mathbf{proj}_{\mathcal{T}_{P}\mathcal{OB}} (Q) = Q - \langle Q, P \rangle_d \odot P
\end{equation}
and the normalization operator, which maps a vector $A$ in the ambient space back to the Oblique manifold as
\begin{equation}
    \mathcal{N}_{\mathcal{OB}}(A) = A \oslash \| A \|_{2, d}
\end{equation}
Both operations are fully supported by modern machine learning frameworks, such as TensorFlow and PyTorch. % \xie{We need to explicitly define all mathematical notations, including $\oslash$, $\odot$, and $\langle\cdot, \cdot \rangle_d$ }

When integrating the Oblique manifold into the update rule (Eq.~\ref{eq:mano-update-rule}), we observe that enforcing the manifold constraint via column-wise normalization assumes that column directions dominate row directions in LLM parameter matrices. However, such an assumption remains unvalidated; for instance, the Muon optimizer demonstrates that all spectral directions hold comparable importance to model convergence. To address the potential insufficiency of the standard Oblique manifold, we introduce a \textbf{rotational manifold} scheme. The reformed approach alternates between column-wise normalization on odd iterations and row-wise normalization on even iterations. By consistently applying this rotation to both the tangent space of the parameters and the update step, we effectively create a custom manifold with oscillating orientation across iterations. Integrating this rotational scheme to the Oblique manifold with out update rule yields the \textbf{Mano} optimizer, as detailed in Alg.~\ref{alg.mano}.

\subsection{The Mano Optimizer}

\begin{algorithm}[H]
\caption{The Mano Optimizer}
\label{alg.mano}
\begin{algorithmic}
\REQUIRE Layer Weight $\theta_t \in \mathbb{R}^{m \times n}$, momentum $M_t \in \mathbb{R}^{m \times n}$, learning rate $\eta_t$ at step $t$, momentum coefficient $\mu$, and weight decay coefficient $\lambda$.
\STATE Initialize $M_0 \gets \mathbf{0} \in \mathbb{R}^{m \times n}, t \gets 0$.
\FOR{each step}
    \STATE $g_t \gets \nabla f(\theta_t)$
    \STATE $M_t \gets \mu \, M_{t-1} + g_t$
    \STATE $k \gets t \bmod 2$                                             \hfill \COMMENT{Rotating Manifold}
    \STATE $\hat{\theta}_t \gets \theta_t \oslash \| \theta_t \|_{2, k}$   \hfill \COMMENT{Manifold Normalization}
    \STATE $v_t \gets M_t - \hat{\theta}_t \odot \langle M_t, \hat{\theta}_t \rangle_{k}$       \hfill \COMMENT{Tangent Momentum}
    \STATE $\hat{v}_t \gets v_ t \oslash \| v_t \|_{2, k}$                 \hfill \COMMENT{Manifold Normalization}
    \STATE $\theta_{t+1} \gets \theta_t - \eta_t (0.2 \sqrt{n_k} \, \hat{v}_t + \lambda \theta_t)$ 
\ENDFOR
\end{algorithmic}
\end{algorithm}

% \xie{Add a list or table summarizing the reformed strategies beyond conventional manifold optimization.}

The Mano optimizer can be summarized as \textbf{Manifold optimization with Euclidean descent}, with the reformed strategies outlined as follows:
\begin{itemize}[topsep=1pt, itemsep=2pt, partopsep=1pt, parsep=1pt]
    \item The parameters $\theta_t$ are not constrained on the manifold; the update process follows weight decay and Euclidean descent rather than retraction.
    \item Rotating Oblique manifold instead of a static geometric structure, alternating through each parameter dimension at every time step (Line $5$).
    \item We first compute the tangent momentum via parameter-space manifold projection (Lines $6-7$), then apply a momentum-space manifold constraint to ensure the update step remains on the Oblique surface (Line $8$).
\end{itemize}

In comparison to SGD-M, Mano only adds two-step column-/row-wise normalizations and one-step tangent-space projection, introduces no additional hyperparameters, and requires no problem-specific assumptions from a geometric or differential perspective. Its implementation is highly streamlined, enabling ease of use and practical adoption from standard SGD-momentum. The memory overhead is comparable to SGD-momentum or Muon, halving the footprint of Adam-based optimizers. The computational cost of applying manifold normalization is also greatly reduced as no \texttt{MatMul} operations are involved, in comparison to the Newton-Schulz iteration. 
We proceed to discuss several key aspects of Mano’s implementation and analyze its computational overhead w.r.t. SGD and Muon’s Newton–Schulz iterations.

\textbf{Rotational Manifold Scheme.} We notice that the iterative procedure of alternating row and column normalization, known as the Sinkhorn-Knopp iteration, converges the matrix input to a doubly stochastic matrix \citep{knight2008sinkhorn}. This set of matrices forms a convex manifold that has been widely studied in the context of manifold optimization \citep{douik2019manifold}. Nevertheless, our empirical results show that applying this iterative normalization procedure intermittently and constraining the tangent vectors only to the Oblique manifold serves as an efficient and effective regularization strategy in LLM training. By conducting ablation studies on static or dynamic manifold normalization in Sec ~\ref{sec:ablation-studies}, we hypothesized that our `manifold rotation' strategy guarantees benign properties on the parameters $\theta$ indirectly, leaving sufficient space for future investigation.

\textbf{Consistent Update RMS.} \citet{liu2025muon} proposed to set the update RMS of Muon to the range of $0.2$ to $0.4$ to be similar to that of AdamW, and uses a rescaling factor of $0.2$ in their final implementation \citep{liu2025muon}. We follow this conclusion and use the same rescaling factor of $0.2$ in all of our experiments for sharing hyperparameters and enabling valid comparison among AdamW, Muon, and Mano. All update matrices $v_t \in \mathbb{R}^{m \times n}$ with column-wise normalization theoretically have an update RMS of $\sqrt{1/m}$ (row-wise normalization with update RMS of $\sqrt{1/n}$). Therefore, we set the rescaling variable of Mano to $0.2\sqrt{n_k}$, for the dimension $n_k \in \{m, n\}, n_0=m, n_1=n$.

\subsection{Theoretical Analysis}

\textbf{Computational Overhead in FLOPs.} For each matrix parameter $\theta \in \mathbb{R}^{m \times n}$, the Mano optimizer computes two column-wise normalization on the parameters $\theta_t$ and the update vectors $v_t$, each requiring $3mn$ FLOPs, which is identical to row-wise normalization. The tangent space projection consumes at most $5mn$ FLOPs due to no \texttt{MatMul} operations being involved. Therefore, the theoretical FLOPs of Mano's update rule are at most $11mn$. For the baseline amount of FLOPs being $6mnB$ for the number of inputs $B$ passed through the layer, the FLOP overhead of Mano is at most $11/6B$, which is consistent for LLMs of different dimensions. In comparison to Muon's FLOP overhead of $5m/B$ \citep{jordan2024muon}, the computational cost of Mano can be neglected in LLM training. 

\textbf{Convergence Analysis.} Theorem \ref{theorem-mano-simple} proves that Mano has convergence guarantees under common assumptions and a simplified setting with no momentum and a static Oblique manifold. The proof is relegated to Appendix \ref{apx:convergence-proof}.
\begin{theorem}[Convergence of Mano w/o Momentum]
    \label{theorem-mano-simple}
    Assume that $f(\theta)$ is an $L$-smooth function, $f$ is lower bounded as $f(\theta) \geq f_{\inf}$, $\mathbb{E}[\xi]=0$ for gradient noise $\xi$ of sub-sampling, $\sin(\phi_t^{(j)}) \geq \gamma > 0$ for angle $\phi_t^{(j)}$ between $g_t^{(j)}$ and the parameter $\theta_t^{(j)}$ and the tangential component $\gamma$. Let Mano run for $T+1$ iterations. If $\eta \leq \frac{C}{\sqrt{T+1}}$ and $m$ equals column dimension size, we have
    \begin{equation}
        \min_{t=0,\ldots,T} \mathbb{E}[\| \nabla f(\theta_t) \|^2] \leq \frac{1}{\sqrt{T+1}} (C_1 + C_2),
    \end{equation}
    where $C_1 = \frac{f(\theta_0) - f_{\inf}}{m^{\frac{1}{2}} \gamma C}, \; C_2 = \frac{L m^{\frac{3}{2}} C}{2\gamma}$.
\end{theorem}

\section{Experiments}

\subsection{Experiment Setup}
\label{sec:experiment_setup}

In this paper, we studied the pretraining performance of five popular models of two class of architectures, including LLaMA-$\{130\text{M}, 350\text{M}, 1.3\text{B}\}$ and Qwen3-$\{0.6\text{B}, 1.7\text{B}\}$, and two common text corpus, including \texttt{C4} and \texttt{Pile} \citep{raffel2020exploring, gao2020pile}.
We utilized a total batch size of $512$ and evaluated models at $10000$ training steps following the experimental setup described in \citet{zhao2024galore} and \citet{raffel2020exploring}. We also report experimental results for the LLaMA-$130$M and -$350$M models, trained on $10$B tokens from the \texttt{Pile} dataset, surpassing the Chincilla optimal scaling law recommendations \citep{hoffmann2022training}.  
% For comparison, we consider Muon and AdamW as the baseline optimizers. 
We set $(\beta_1, \beta_2)=(0.9, 0.95)$ for AdamW and the same momentum coefficient $\mu=0.95$ for Mano and Muon, with the other hyperparameters available in Appendix \ref{apx:hyperparameters}. %Owing to computational constraints, the empirical scope of this study remains limited. We plan to include more extensive evaluations on larger LLMs and text corpora in future revisions to validate the scalability of Mano.

\subsection{Experiment Results}

\begin{figure}[ht]
    \centering
    \subfigure[LLaMA-350M / \texttt{C4}]{\includegraphics[width=0.48\linewidth]{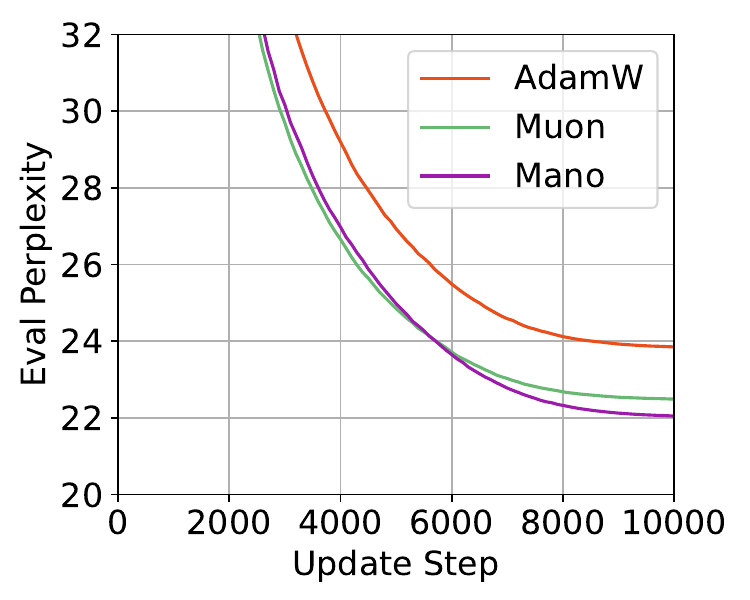}}
    \subfigure[LLaMA-1.3B / \texttt{C4}]{\includegraphics[width=0.48\linewidth]{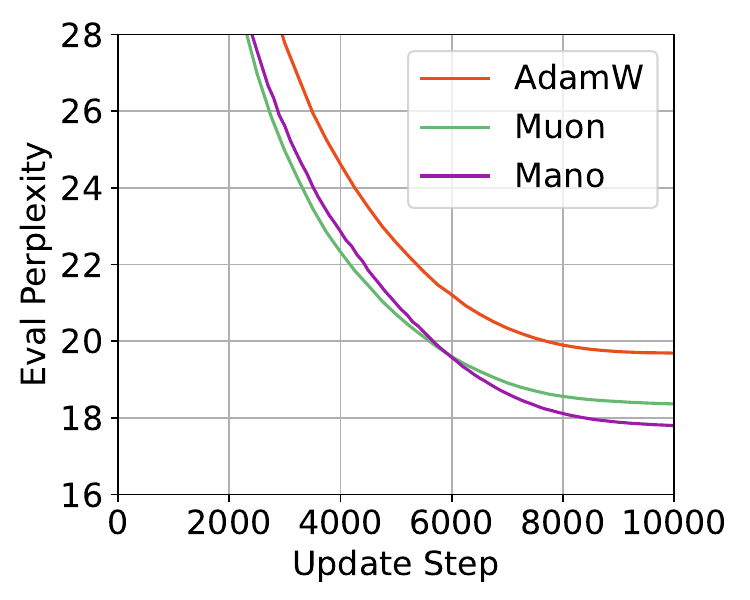}}
    \subfigure[LLaMA-350M / \texttt{Pile}]{\includegraphics[width=0.48\linewidth]{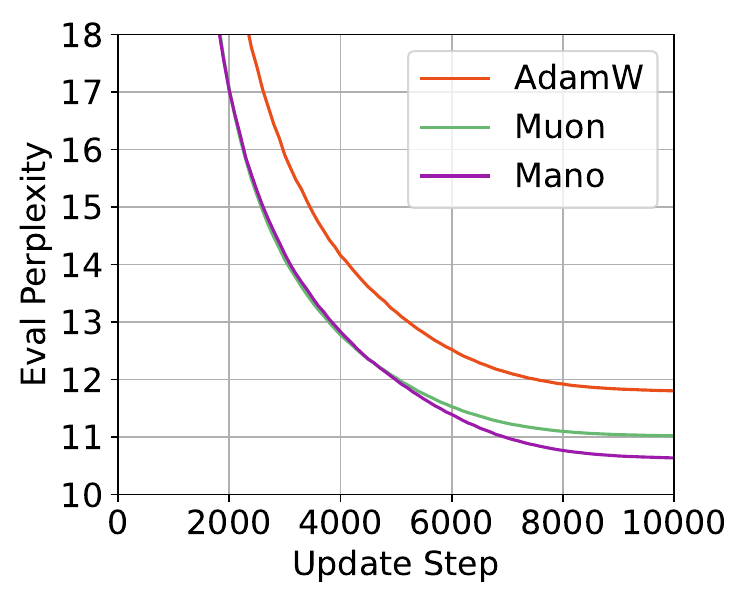}}
    \subfigure[LLaMA-1.3B / \texttt{Pile}]{\includegraphics[width=0.48\linewidth]{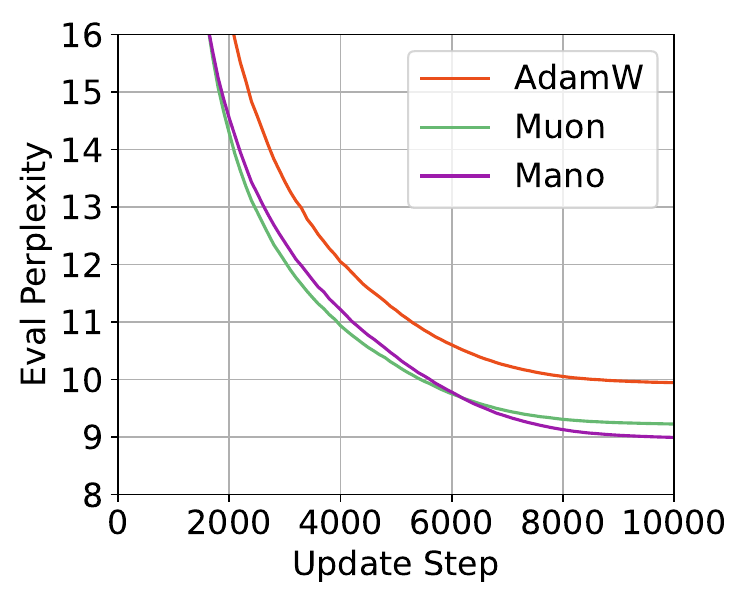}}
    \caption{LLaMA-350M and -1.3B models trained on the \texttt{C4/en} and \texttt{Pile} dataset for $10000$ steps with three different optimizers: AdamW, Muon, and Mano. Mano demonstrated a faster convergence speed than both popular optimizers with the simplest implementation and computational cost. }
    \label{fig:LLaMA-perplexity}
\end{figure}

We present the pretraining dynamics of LLMs in test perplexity and compare the performance on Mano to the baseline optimizers AdamW and Muon. Fig.~\ref{fig:LLaMA-perplexity} reports a consistent advantage in sample efficiency of Mano of two LLaMA models on the \texttt{C4} and \texttt{Pile} dataset. We observe that Mano exhibits a distinct convergence pattern: though its initial convergence may be slower than that of Muon, its loss reduction in the later stages is surprisingly faster than both AdamW and Muon. While the loss curves of the two baseline optimizers plateau, Mano continues to progress at a nearly constant rate toward the global minimum and ultimately surpasses Muon, which may suggest that Mano is more effective at escaping local minima. We also observe that, for larger models, the point at which Mano’s loss descent rate surpasses that of Muon occurs later, potentially due to their larger data-scaling optima. 

\begin{figure}[ht]
    \centering
    \subfigure[Qwen3-0.6B / \texttt{Pile}]{\includegraphics[width=0.48\linewidth]{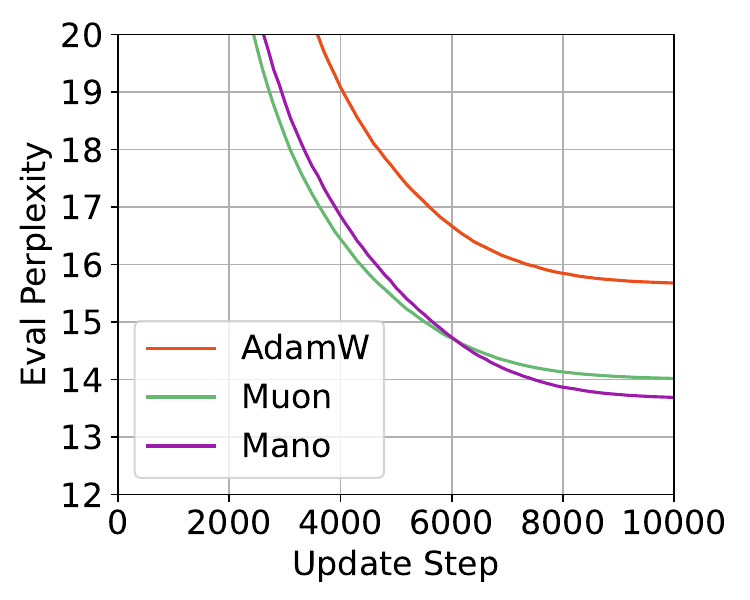}}
    \subfigure[Qwen3-1.7B / \texttt{Pile}]{\includegraphics[width=0.48\linewidth]{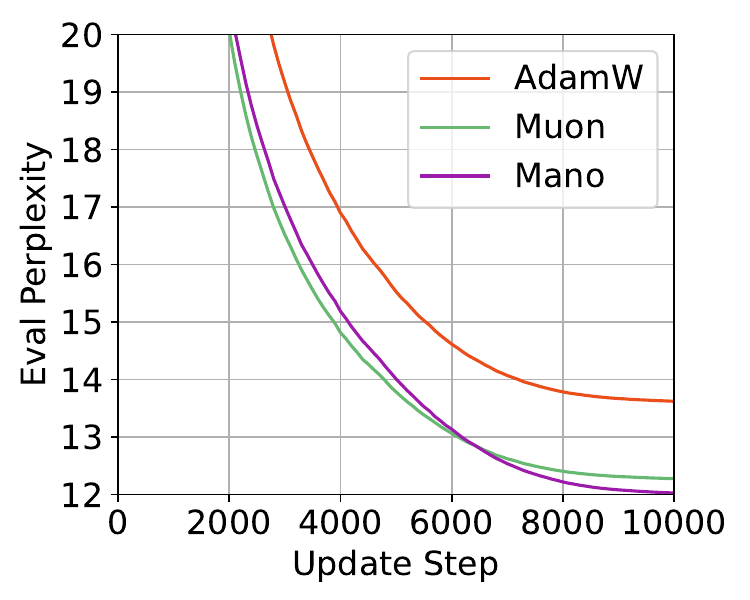}}
    \caption{Qwen3-0.6B and -1.7B models trained on the \texttt{Pile} dataset for $10000$ steps with three different optimizers: AdamW, Muon, and Mano. The performance advantage of Mano is model-transferrable. }
    \label{fig:Qwen3-perplexity}
\end{figure}

Fig.~\ref{fig:Qwen3-perplexity} reports the replicated experiments on the Qwen3 architecture and the \texttt{Pile} dataset, demonstrating that the performance advantage of Mano can be transferred across different model architectures. 
% The difference in training dynamics of Qwen3-$0.6$B and -$1.7$B models may be attributable to differences in the maximum learning rate ($6.0 \times 10^{-4}$ and $3.0 \times 10^{-4}$), suggesting the necessity of tuning the learning rate scheme. Nevertheless, Mano's loss curve on Qwen3-$0.6$B closely follows a tangent to that of AdamW, enabling a direct comparison of their convergence behavior. 
As Mano consistently improves the sample efficiency of pretraining LLMs in comparison to Muon, we hypothesize that projecting and constraining the training trajectory onto a manifold more accurately captures the steepest-descent path in the original solution space, without limiting the expressivity of LLM parameters. We further provide empirical results on the Qwen3-$0.6$B model with different maximum learning rate in Appendix \ref{apx:empirical-designs} to validate Mano's robustness across learning rate settings. 

\begin{table*}[h]
    \centering
    \caption{Numerical result of the final test perplexity of LLMs trained by different optimizers on the two pretraining corpora \texttt{C4} and \texttt{Pile} for $10000$ update steps and consistent hyperparameters. Mano yields consistent gains in sample efficiency.}
    \begin{tabular}{|c|cc|cccc|}
        \hline
        Datasets & \multicolumn{2}{c|}{\texttt{C4/en}} & \multicolumn{4}{c|}{\texttt{Pile}} \\
        \hline
        Models     & Llama-350M & Llama-1.3B & Llama-350M & Llama-1.3B & Qwen3-0.6B & Qwen3-1.7B \\
        \hline \hline
        AdamW      & 23.852 & 19.690 & 11.803 & 9.945 & 15.679 & 13.624 \\
        Muon       & 22.491 & 18.365 & 11.022 & 9.227 & 14.020 & 12.276 \\
        Mano       & \textbf{21.182} & \textbf{17.800} & \textbf{10.549} & \textbf{8.994} & \textbf{13.689} & \textbf{12.028} \\
        \hline
    \end{tabular}
\end{table*}

%\textbf{Over-train Settings.} 
We further provide the experiment results with more training tokens. Due to computational constraints, we train LLaMA-130M and -350M models on the \texttt{Pile} dataset for $10B$ tokens, with results provided in Fig.~\ref{fig:LLaMA-perplexity-overtrain}. We notice that Mano performed worse than AdamW in the middle of training the LLaMA-$130$M model, but ultimately achieved the best performance across the three optimizers. We are expecting to expand these over-trained experiments to bigger LLMs and further understand this intriguing loss descent pattern of Mano in the later convergence stage.

\begin{figure}[ht]
    \centering
    \subfigure[LLaMA-130M / \texttt{Pile}]{\includegraphics[width=0.49\linewidth]{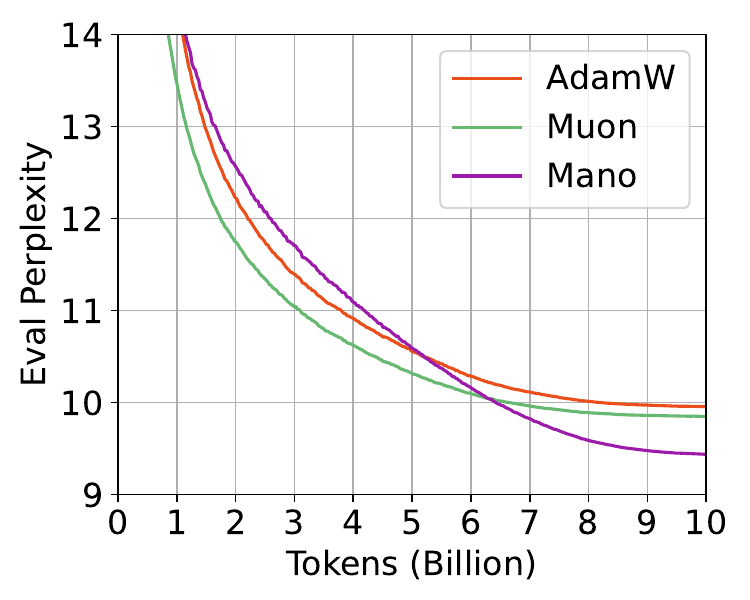}}
    \subfigure[LLaMA-350M / \texttt{Pile}]{\includegraphics[width=0.49\linewidth]{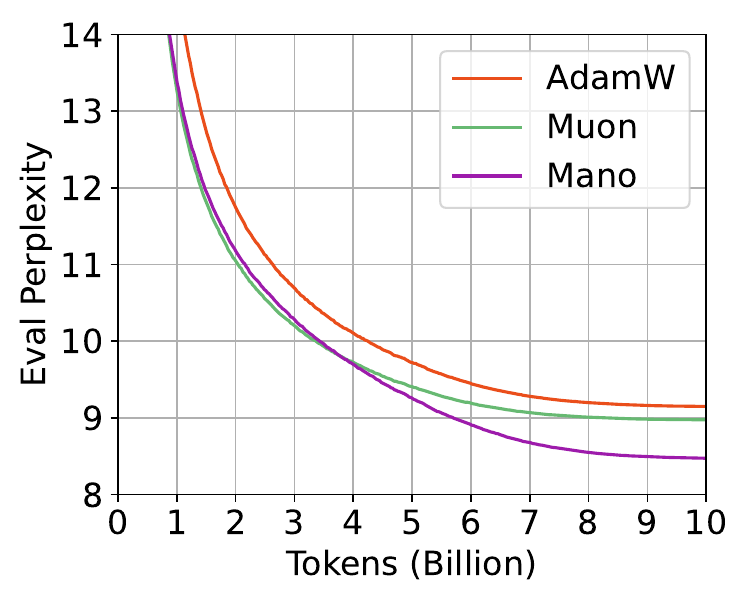}}
    \caption{LLaMA-$130$M and -$350$M models trained on the \texttt{Pile} dataset for $10$B tokens. We demonstrated that with data scaling, Mano consistently performed better than Muon and AdamW in the ultimate convergence speed. }
    \label{fig:LLaMA-perplexity-overtrain}
\end{figure}

\subsection{Learning Dynamics}

In this subsection, we will delve into the learning dynamics of the Mano optimizer and compare it to the baseline optimizers AdamW and Muon from multiple perspectives.

\textbf{Gradient Stability.} To understand the internal training dynamics of Mano, we reported the average gradient norm, variance, and the Signal-to-Noise (SNR) ratio in Fig.~\ref{fig:LLaMA-gradient-norm}. Our empirical observations reveal the intrinsic advantage for Mano, that it consistently maintains a lower gradient variance compared to Muon, when operating under the same momentum coefficient and a similar update RMS. The SNR of Mano is notably higher than of Muon, suggesting superior training stability. We hypothesize that Mano’s manifold normalization approach preserves the essential curvature information encoded within the original gradient step and promotes a more stable optimization landscape. This relationship is further evidenced by the spectral distributions of both optimizers, discussed in the following paragraph.

\begin{figure*}[ht]
    \centering
    \subfigure[Gradient Norm]{\includegraphics[width=0.28\linewidth]{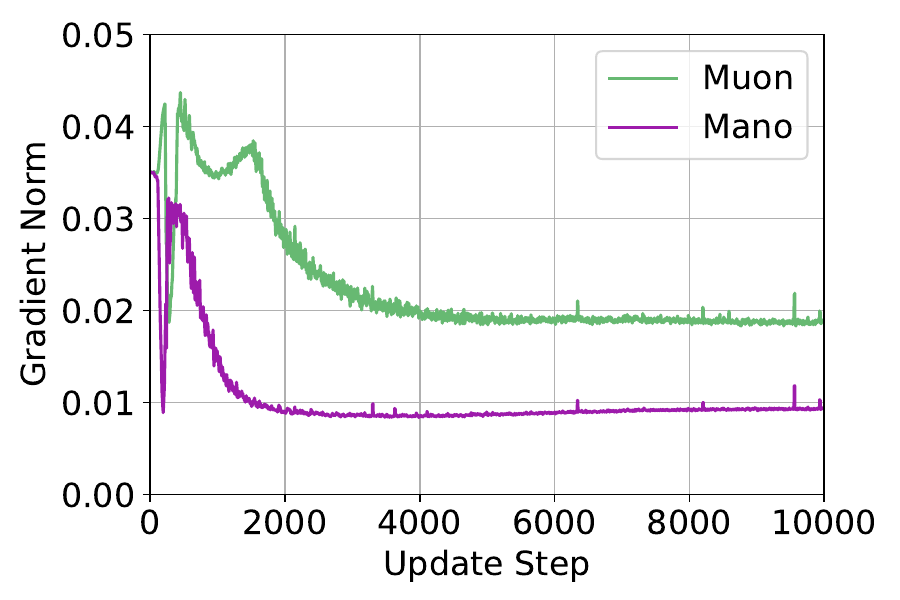}}
    \subfigure[Gradient Variance]{\includegraphics[width=0.28\linewidth]{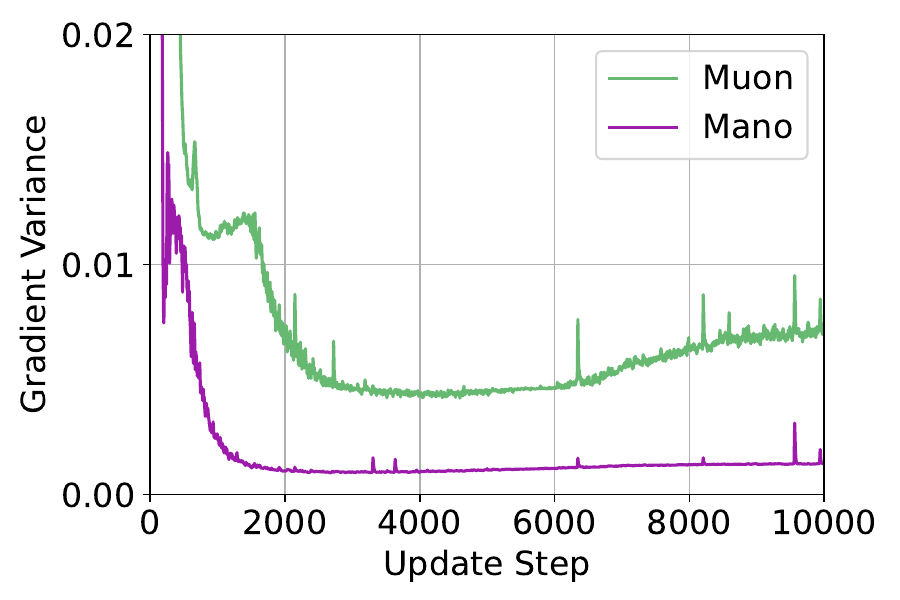}}
    \subfigure[Gradient SNR]{\includegraphics[width=0.28\linewidth]{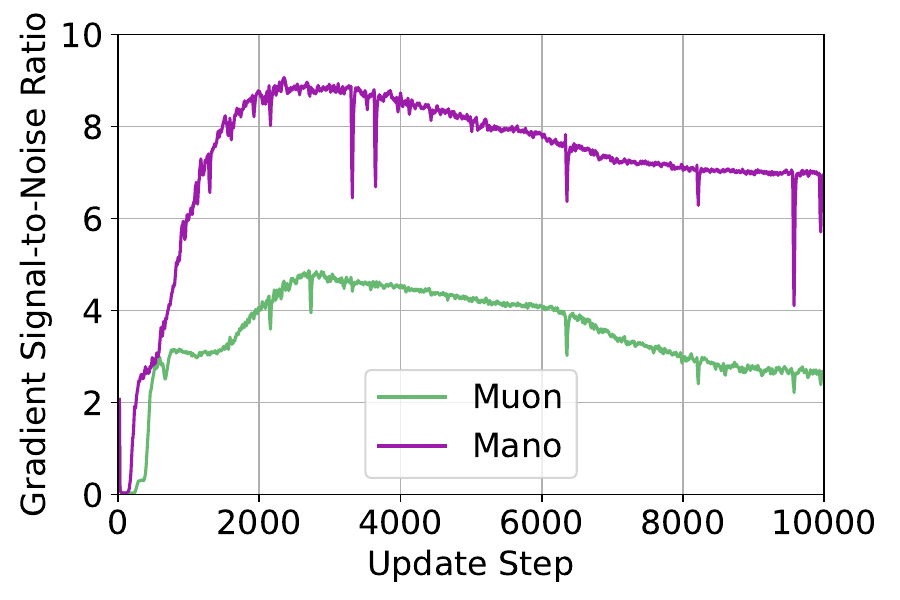}}
    \caption{
        % \xie{Gradient variance is more important. Analyze gradient variance rather than gradient norm variance. Gradient variance reduction is theoretically related to convergence. We can add some theoretical refs and discussion from this perspective.} 
        The average (a) Gradient norm, (b) Gradient variance, and (c) Gradient Signal-to-Noise Ratio (SNR) of LLaMA-$350$M model parameters trained on the \texttt{Pile} dataset. The SNR is calculated as the norm-to-variance ratio. As an indicator of internal training dynamics, Mano exhibits lower gradient variance and a higher SNR than Muon, both under the same momentum coefficient $\mu=0.95$.
    }
    \label{fig:LLaMA-gradient-norm}
\end{figure*}

\textbf{Spectral Distribution.} Spectral preconditioning has attracted widespread interest following the empirical success of Muon. We analyze Mano from a spectral perspective by comparing the spectra of its update matrices with those of AdamW and Muon in Fig.~\ref{fig:LLaMA-update-spectra}. We observe that Mano achieved efficient spectral regularization through manifold normalization that increases the relative magnitude of rare directions with a monotone transformation of singular values in the momentum. While Muon performs whitening and flattens the spectrum, it discards the singular order information, which can be suboptimal from a theoretical perspective, as suggested by \citet{su2025isotropic}. %\xie{Why not use citet?}

\begin{figure}[ht]
    \centering
    \subfigure[Attention Layer]{\includegraphics[width=0.48\linewidth]{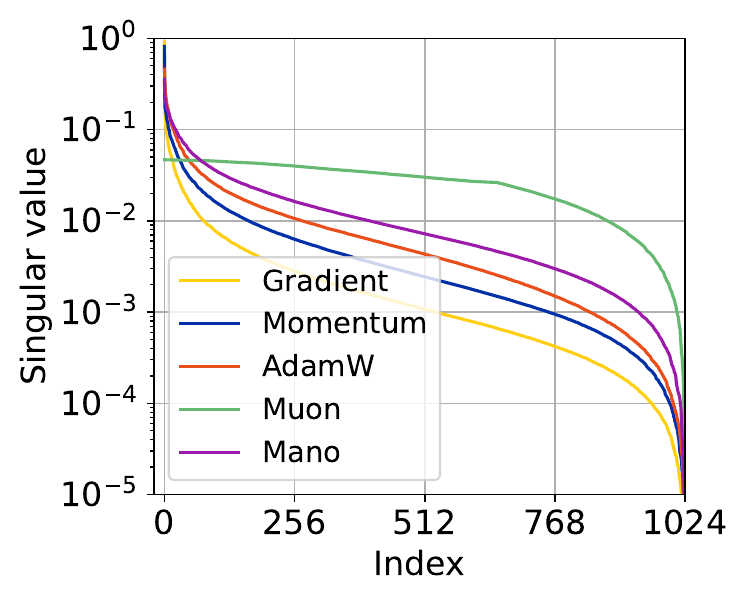}}
    \subfigure[MLP Layer]{\includegraphics[width=0.48\linewidth]{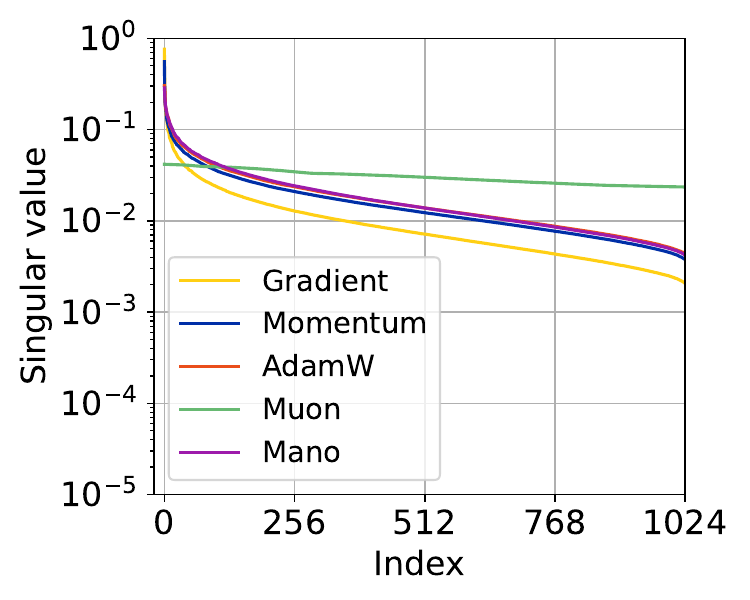}}
    \caption{The spectral distributions of (a) all attention layers and (b) all MLP layers from an LLaMA-$350$M model at the $1000$ step on the \texttt{C4/en} corpus, including the model gradient, momentum, and the update matrix of AdamW, Muon, and Mano. The manifold normalization of Mano may also be viewed as an efficient spectral regularization method that lifted the update spectra while preserving the singular values' original ordering. }
    \label{fig:LLaMA-update-spectra}
\end{figure}

\textbf{Wall-clock Time Comparison.} We have previously derived the theoretical FLOPs overhead of Mano, that without \texttt{MatMul} operations, the computational cost of the proposed manifold normalization operation is neglectable for LLM training. 
To assess the practical computational efficiency, we conduct a performance analysis of the normalization operations used in Mano and Muon, reporting their respective wall-clock times in Tab.~\ref{tab:runtime-analysis}. The observations suggest that the computational time of Mano grows linearly with the increase in the LLM's dimension, in contrast to the exponential growth observed for Muon. Fig.~\ref{fig:LLaMA-wall-clock-perplexity} compares the practical training performance measured in wall-clock time. In an experiment of one-day on the LLaMA-$350$M and -$1$B models, Mano achieves $1.75\times$ and $1.38\times$ faster convergence than Muon with continuously growing advantage.

\begin{table}[ht]
\centering
\caption{Computational cost comparison of Newton-Schulz iteration (NS, $T=5$) and the manifold normalization enforced by Mano on Attention and MLP matrices from LLaMA-$1$B, -$7$B, and -$70$B models in \texttt{BFloat16}. Reported values denote the average over 1000 PyTorch runs, with peak GPU memory usage measured via \textsc{torch.cuda}. Mano incurs significantly lower computational overhead than Muon, both theoretically with constant-time complexity and empirically in LLM experiments.}
\newcommand{\vsp}{\rule{0pt}{10pt}}
\resizebox{0.9\linewidth}{!}{
\begin{tabular}{|c|c|c|c|}
    \hline
    \vsp \textbf{Module} & \textbf{Metric} & \textbf{Newton-Schulz} & \textbf{Mano} \\
    \hline\hline

    \multicolumn{4}{|c|}{\vsp \textbf{LLaMA-1B} (dim=$2048$)} \\ \hline
    \multirow{2}{*}{Attention} & Time & 2.01 (ms) & 0.14 (ms) \\
                               & Mem  & 64.1 (MB) & 56.0 (MB) \\
    \hline
    \multirow{2}{*}{MLP}       & Time & 4.68 (ms) & 0.17 (ms) \\
                               & Mem  & 119.3 (MB) & 87.3 (MB) \\
    \hline\hline
    
    \multicolumn{4}{|c|}{\vsp \textbf{LLaMA-7B} (dim=$4096$)} \\ \hline
    \multirow{2}{*}{Attention} & Time & 14.83 (ms) & 0.34 (ms) \\
                               & Mem  & 224.0 (MB) & 192.0 (MB) \\
    \hline
    \multirow{2}{*}{MLP}       & Time & 30.22 (ms) & 1.45 (ms) \\
                               & Mem  & 472.0 (MB) & 344.0 (MB) \\
    \hline\hline
    
    \multicolumn{4}{|c|}{\vsp \textbf{LLaMA-70B} (dim=$8192$)} \\ \hline
    \multirow{2}{*}{Attention} & Time & 110.79 (ms) & 2.19 (ms) \\
                               & Mem  & 896.0 (MB) & 512.0 (MB) \\
    \hline
    \multirow{2}{*}{MLP}       & Time & 184.33 (ms) & 4.35 (ms) \\
                               & Mem  & 1536.0 (MB) & 1024.0 (MB) \\
    \hline
\end{tabular}
}
\label{tab:runtime-analysis}
\end{table}

\subsection{Ablation Studies}
\label{sec:ablation-studies}

In this subsection, we will present various ablation study results on the critical design choices of Mano and Manifold optimization methods, aiming to provide a complete understanding of how our strategy advanced.

\textbf{Riemannian SGD-M.} We first compare Mano and our reformed strategies to standard Riemannian SGD with momentum (RSGD-M) on the Oblique manifold. While the implementation of the two optimizers shares many similarities, their performance in training LLMs diverges significantly. As illustrated in Fig.~\ref{fig:LLaMA-rsgd-m}, standard RSGD-M struggles to optimize the LLaMA-$350$M model, failing to reach the optimal loss range of $2.0$ to $3.0$ or show any sign of convergence. In contrast, Mano can significantly reduce the loss beyond RSGD-M. We posit that because traditional manifold methods rely on retractions to map the parameters onto smooth surfaces, they constrain LLMs' expressivity and hinder exploration of the loss landscape. By avoiding the assumption that the objective or solution must reside on a specific matrix manifold, Mano enables more flexible training dynamics necessary for LLMs.

\begin{figure}[ht]
    \centering
    \includegraphics[width=0.85\linewidth]{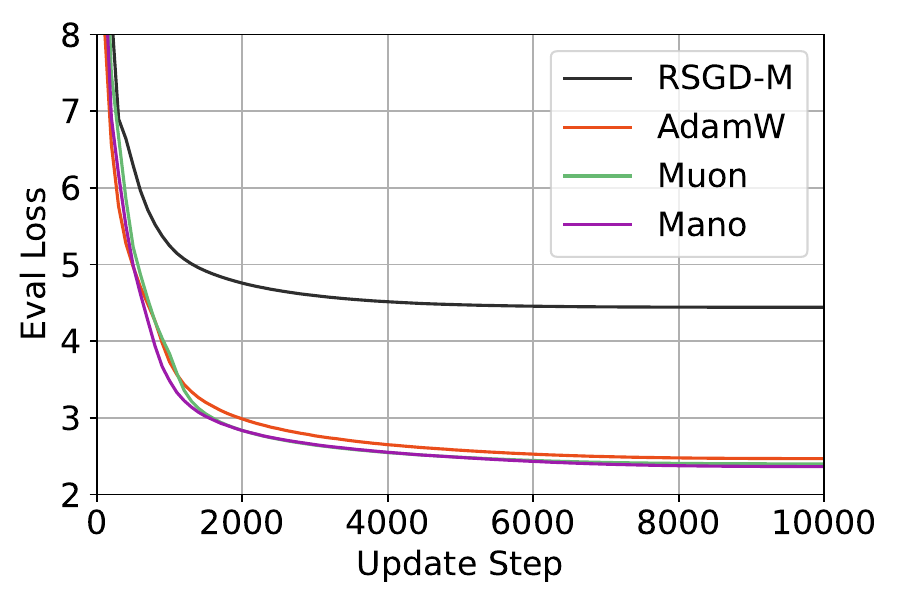}
    \caption{Comparing conventional Riemannian SGD-M and Mano on LLaMA-$350$M models trained on the \texttt{Pile} dataset. Unlike traditional manifold optimization methods that impose constraints on model parameters and expressivity during training, Mano provides a more flexible approach and superior performance.}
    \label{fig:LLaMA-rsgd-m}
\end{figure}

\textbf{Dynamic or Static Oblique Manifold?} A key feature of Mano's implementation is the rotational Oblique manifold scheme. To understand how this feature functioned in optimization, we provide ablation experiment results in Tab.~\ref{tab:LLaMA-ablation-perplexity} with a static Oblique manifold fixed at the $0$th dimension for parameters and update steps. We show that, under fixed column-wise normalization, Mano achieves comparable test perplexity on LLaMA-350M but performs significantly worse on LLaMA-1B, indicating a poor model-wise scaling behavior.

\begin{table}[ht]
    \centering
    \caption{Ablation results on LLaMA-$350$M and -$1.3$B models' trained on the \texttt{Pile} dataset, with test perplexity reported at the $10000$ steps. While the static Oblique manifold hindered LLaMA-$1.3$B performance, momentum retraction yielded performance gains on LLaMA-$350$M, leaving space for further investigation into scale-dependent behavior.}
    \newcommand{\vsp}{\rule{0pt}{10pt}}
    \begin{tabular}{|c|c|c|}
        \hline
        \vsp Eval Perplexity & LLaMA-350M & LLaMA-1B \\
        \hline \hline
        AdamW       & 11.803 & 9.945 \\
        Muon        & 11.022 & 9.227 \\
        \hline        
        Mano (Static $\mathcal{M}$) & 10.684 & 9.254 \\
        Mano ($M_t=v_t$) & 10.540 & 8.998\\
        Mano        & 10.549 & 8.994 \\
        \hline
    \end{tabular}
    \label{tab:LLaMA-ablation-perplexity}
\end{table}

\textbf{Momentum with or without Retraction?} Combining the traditional Manifold optimization strategies and our reformulated Mano optimizer, we examined whether the manifold constraints are extended from the update steps to the buffered momentum as well. By simply replacing $v_t$ to $M_t$ in Alg.~\ref{alg.mano}, namely $M_t=v_t$, the momentum buffer is updated as the tangent momentum. Results in Tab.~\ref{tab:LLaMA-ablation-perplexity} suggest essentially identical results for both LLaMA-$350$M and -$1$B models; additional experiments are required to fully validate this design.
%However, retracting the momentum on the tangent space significantly changes the theoretical property of Mano. Given that Alg.~\ref{alg.mano} yields notable convergence acceleration within standard Euclidean momentum frameworks, this study focused on studying the standard SGD-M variant to isolate its performance benefits. 
Ultimately, the integration of manifold optimization into modern frameworks offers a vast design space. While this study cannot exhaustively explore every aspect, Mano introduces a compelling design philosophy with the potential to redefine optimization rules in high-dimensional regimes.

\section{Conclusion}

\textbf{Limitations.} The empirical scope of this work is constrained by available computational resources. We aim to leave additional experiments, such as hyperparameter fine-tuning and over-training experiments of larger models, as future works.
%As a result, several experiments are not included, such as systematic hyperparameter tuning for Mano (e.g., momentum and learning rate), over-training experiments of larger-scale language models, repeated runs to quantify performance variance, and wall-clock time comparisons under different distributed training configurations. 
On the theoretical side, we frankly note that our current convergence analysis holds for a simplified version of the Mano optimizer, while recent LLM training methods often lacked convergence analysis. Extending the theory to cover momentum dynamics and broader optimization regimes remains an important future direction.

\textbf{Summary.} To the best of our knowledge, this is the first study to reformulate manifold optimization methods for efficient training LLMs. The proposed optimizer, Mano, departs from traditional manifold optimization techniques and modern optimizers that perform spectral preconditioning or second-moment estimates. Empirical results demonstrate that Mano outperforms the existing baseline of AdamW and Muon in training LLMs with significantly lower computational overhead than Muon and a reduced memory footprint compared to AdamW. Based on the hypothesis that mapping learning trajectories to smooth manifold surfaces can accelerate training convergence, this study highlights the potential of utilizing geometrically aware manifold techniques in conjunction with modern optimization strategies.

\newpage
\section*{Acknowledgment}
We gratefully thank Juanxi Tian for early discussions and hypotheses contributed to this work.

\bibliography{references}
\bibliographystyle{icml2026}

%%%%%%%%%%%%%%%%%%%%%%%%%%%%%%%%%%%%%%%%%%%%%%%%%%%%%%%%%%%%%%%%%%%%%%%%%%%%%%%
%%%%%%%%%%%%%%%%%%%%%%%%%%%%%%%%%%%%%%%%%%%%%%%%%%%%%%%%%%%%%%%%%%%%%%%%%%%%%%%
% APPENDIX
%%%%%%%%%%%%%%%%%%%%%%%%%%%%%%%%%%%%%%%%%%%%%%%%%%%%%%%%%%%%%%%%%%%%%%%%%%%%%%%
%%%%%%%%%%%%%%%%%%%%%%%%%%%%%%%%%%%%%%%%%%%%%%%%%%%%%%%%%%%%%%%%%%%%%%%%%%%%%%%
\newpage
\appendix
\onecolumn

\section{Impact Statement}

This paper aims to advance the understanding of deep learning optimization. While our findings may contribute to improving the sustainability of LLM training and, more broadly, societal welfare and environmental outcomes, we do not discuss specific social impact in this work.

\section{Details for Reproducibility}
\label{apx:hyperparameters}

\subsection{Hyperparameters}

In this paper, we follow the experimental setup described in \citet{zhao2024galore} and \citet{raffel2020exploring}. The model architecture and respective hyperparameters are presented in Tab.~\ref{tab:hyperparameters}. Besides the learning rate and batch size settings, we use a cosine decay learning rate scheduler with a minimum learning rate ratio of $0.1$ for all experiments. Weight decay is set to $0.1$, and gradients are clipped at $1.0$. The LLaMA models are tokenized using the T5 tokenizer, and the Qwen3 models use the generative Qwen3 tokenizer. For optimizer hyperparameters, we use the default $(\beta_1, \beta_2) = (0.9, 0.95)$ for AdamW, number of Newton-Schulz iterations $T=5$ for Muon, and the momentum coefficient $\mu=0.95$ for both Muon and Mano. All training is performed in \textit{BFloat16} mixed precision. All experiments are conducted with data distributed parallel (DDP) on $4 \times$ NVIDIA H800-PCle-80G GPUs, except for the LLaMA-130M experiments, which are performed on $4 \times$ NVIDIA RTX-4090 GPUs. 

\begin{table}[H]
    \centering
    \renewcommand{\arraystretch}{1.2}
    %\resizebox{\textwidth}{!}{
    \begin{tabular}{|c|ccc|cc|}
        \hline
        Model & LLaMA-130M & LLaMA-350M & LLaMA-1.3B & Qwen-0.6B & Qwen-1.7B \\
        \hline \hline
        Layer num           & 12    & 24    & 24    & 28    & 28    \\
        Hidden dim size     & 768   & 1024  & 2048  & 1024  & 2048  \\
        FFN dim size        & 1024  & 2736  & 5461  & 3072  & 6144  \\
        Attention heads     & 12    & 16    & 32    & 16    & 16    \\
        \hline
        Seq-len & 1024 & 1024 & 1024 & 1024 & 1024 \\
        \hline
        Max Learning Rate & $6.0\times10^{-4}$ & $3.0\times10^{-4}$ & $3.0\times10^{-4}$ & $3.0\times10^{-4}$ & $3.0\times10^{-4}$ \\
        \hline 
        Batch Size & \multicolumn{5}{c|}{16} \\
        GradAcc & \multicolumn{5}{c|}{8} \\
        \hline
        Total Batch Size & \multicolumn{5}{c|}{512} \\
        \hline
        Iterations & \multicolumn{2}{c|}{10000 / 36000} & \multicolumn{3}{c|}{10000} \\
        \hline
        Warmup iterations & \multicolumn{5}{c|}{1000} \\
        \hline
    \end{tabular}
    %}
    \caption{Training configurations for different LLaMA and Qwen model scales, including architecture details, sequence length, learning rate, batch size and gradient accumulation steps, and training schedule. For the over-trained LLaMA-$130$M and -$350$M models with $10$B training corpus, the training iterations are extended to $36000$ with all other hyperparameters unchanged.}
    \label{tab:hyperparameters}
\end{table}

\subsection{Additional Empirical Designs for Mano}
\label{apx:empirical-designs}

In this section, we report the complete hyperparameter configurations to facilitate reproducibility and further discuss the empirical designs for Mano.

\textbf{Nesterov Accelerated Gradient.} Empirical studies demonstrated that Nesterov-style momentum performs better than normal SGD-momentum for Muon, thus it has been made the default in the public implementation of Muon \citep{jordan2024muon, liu2025muon}. In our experiments of Mano, Nesterov Accelerated Gradient (NAG) may yield better performance than standard momentum for large-scale models, but can occasionally degrade it for smaller models, while having an overall minor effect on the ultimate training trajectory. For consistency, all experiment results in this paper use the standard momentum implementation and include NAG as an option in our implementation.

\textbf{Input and Output Parameters.} The update rule of Mano may apply to parameters of arbitrary dimensionality, for which the rotational manifold scheme can be implemented by iteratively traversing the Oblique manifold along each dimension. However, we followed the implementation of Muon to optimize the LLM's input and output parameters and $1-D$ bias using AdamW \citet{jordan2024muon}. The modular norm theory stated that the optimization dynamics of the embedding layer should be different from other layers, which applies to the lm head layer as well, according to empirical studies of Muon \citep{large2024scalable}. We hypothesize that the structural properties of the input embedding and output head layers in LLMs are constrained by the high sparsity of vocabulary activations, such that neither matrix orthogonalization nor manifold normalization outperforms AdamW adaptive learning per-parameter.

\textbf{Learning Rate Independence.} For all models, we used a uniform learning-rate schedule across experiments. Although different optimization algorithms may have different optimal learning rates, we controlled for this factor by constraining the RMS magnitude of parameter updates to AdamW's range of $0.2$ to $0.4$, as proposed by \citet{liu2025muon} and discussed in the main paper. This normalization ensures that the optimizers operate at similar effective step sizes, allowing a fair comparison of their optimization behavior without optimizer-specific learning-rate tuning. Fig.~\ref{fig:Qwen3-learning-rate} shows the training performance of Qwen3-$0.6$B models on the \texttt{Pile} dataset under a fixed learning-rate schedule with varying maximum learning-rate values. We find that Mano converges more slowly than Muon and AdamW during the early training phase, but achieves faster convergence in the later stage when using a higher learning rate. Despite the different learning-rate value, Mano achieves a noticeably lower final test perplexity than both baseline optimizers across the two experiments.

\begin{figure}[ht]
    \centering
    \subfigure[Qwen3-$0.6$B / LR=$3.0 \times 10^{-4}$]{\includegraphics[width=0.33\linewidth]{images/pile_qwen3_0.6b_final_eval_perplexity_update_step.pdf}}
    \subfigure[Qwen3-$0.6$B / LR=$6.0 \times 10^{-4}$]{\includegraphics[width=0.33\linewidth]{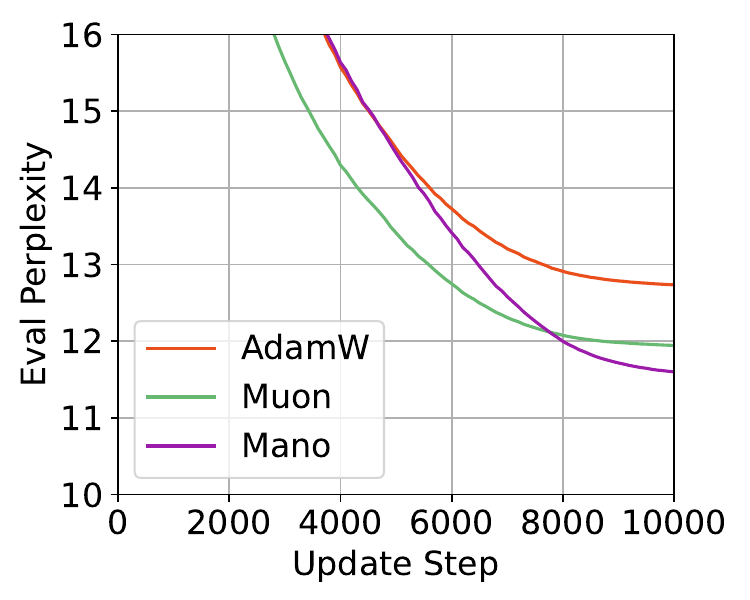}}
    \caption{We train Qwen3-$0.6$B models on the \texttt{Pile} dataset for $10000$ steps using different baseline learning rates and the same learning-rate schedule. We observe that for a higher learning rate,
    }
    \label{fig:Qwen3-learning-rate}
\end{figure}

\section{Mano for General Tensor}

We provided the Mano optimizer in its general form for order-$d$ tensor in Alg.~\ref{alg.mano-general-tensor}. At step $t$, the tangent vector projection and manifold normalization are applied to the $t \bmod d$-th dimension of the parameters and the update step. 

\begin{algorithm}[H]
\caption{The Mano Optimizer for general tensor}
\label{alg.mano-general-tensor}
\begin{algorithmic}[1]
\REQUIRE Weight $\theta_t \in \mathbb{R}^{n_1\times \cdots \times n_k}$, momentum $M_t \in \mathbb{R}^{n_0\times \cdots \times n_{d-1}}$, learning rate $\eta_t$ at step $t$, momentum coefficient $\mu$, and weight decay coefficient $\lambda$.
\STATE Initialize $M_0 \gets \mathbf{0} \in \mathbb{R}^{n_0\times \cdots \times n_{d-1}}, t \gets 0$.
\FOR{each step}
    \STATE $g_t \gets \nabla f(\theta_t)$
    \STATE $M_t \gets \mu \, M_t + g_t$
    \STATE $k \gets t \bmod k$
    \STATE $\hat{\theta}_t \gets \theta_t \oslash \| \theta_t \|_{2, k}$ 
    \STATE $v_t \gets M_t - \hat{\theta}_t \odot \langle M_t, \hat{\theta}_t \rangle_{k}$
    \STATE $\hat{v}_t \gets v_t \oslash \| v_t \|_{2, k}$
    \STATE $\theta_{t+1} \gets \theta_t - \eta_t (0.2 \sqrt{n_k} \, \hat{v}_t + \lambda \theta_t)$
\ENDFOR
\end{algorithmic}
\end{algorithm}

% \section{Additional Experiment Results}

\section{Relationship to Existing Optimizers}
\label{apx:other-optimizers}

\textbf{Adafactor:} By employing a factored rank-1 approximation of the second-order moments, Adafactor is closely related to Mano, which both explicitly normalize updates along parameter dimensions than globally or per-coordinate \citep{shazeer2018adafactor}. However, Adafactor relies on EMA-based second-moment normalization and controls per-parameter RMS in a manner similar to Adam, whereas Mano enforces manifold-based normalization and explicitly constrains the update steps to lie on an Oblique manifold. Consequently, Mano regulates update magnitudes geometrically rather than statistically, yielding stronger stability guarantees.

\textbf{Spectral Optimizers:} Mano fundamentally differs from existing spectral optimizers (e.g., Shampoo, SOAP, Muon, Conda) in that it does not rely on matrix-wide spectral information or second-order preconditions, but instead performs only vector-based operations to apply geometric constraints at each step. Although the Oblique manifold admits a spectral interpretation, we view Mano as a computationally efficient alternative to the current paradigm of extensive spectral preconditioning employed in optimization.

\textbf{SSO and Hyperball:} Two recent optimizers have integrated manifold constraints to accelerate LLM pretraining. The Hyperball optimizer \citep{wen2025hyperball} employs manifold normalization to regulate the effective step sizes and weight norms, serving as an effective alternative to the weight decay scheme. Similarly, the Spectral Sphere Optimizer (SSO) \citep{xie2026controlled} constrains both model weights and updates to a spectral sphere, aligning with maximal update parameterization ($\mu P$) \citep{yang2023spectral}. Our proposed methodology departs from both approaches: Mano applies manifold normalization to the momentum and requires no spectral preconditioning.

\section{Proofs}
\label{apx:convergence-proof}

We provide a proof of convergence for the following simplified update rule of Mano, which excludes the momentum, and fixes the Oblique manifold at the $0$-th dimension (with dimension size $m$).
\begin{equation}
    \label{eq:mano-wo-momentum}
    \begin{cases}
        g_t \gets \nabla f(\theta_t) \\
        \hat{\theta}_t \gets \theta_t \oslash \| \theta_t \|_{2, 0} \\
        v_t \gets g_t - \hat{\theta}_t \odot \langle g_t, \hat{\theta}_t \rangle_{0} \\
        \hat{v}_t \gets v_t \oslash \| v_t \|_{2, 0} \\
        \theta_{t+1} \gets \theta_t - \eta \sqrt{m} \hat{v}_t
    \end{cases}
\end{equation}
The mathematical operations involved are defined as follow:
\begin{itemize}
    \item Element-wise product ($\odot$): $P \odot Q \triangleq \bigl( P_{ij} Q_{ij} \bigr)_{i,j}$.
    \item Element-wise division ($\oslash$): $P \oslash Q \triangleq \bigl(P_{ij} / Q_{ij} \bigr)_{i,j}$.
    \item Dimension-wise inner product ($\langle \cdot,\cdot \rangle_k$): For $j \in \{0, \ldots, n_k-1$\} and the $k$-th dimension, the $j$-th component $\langle Q, P \rangle_d^{(j)} = \langle Q^{(j)}, P^{(j)} \rangle$.
    \item Dimension-wise norm ($\| \cdot \|_{2,k}$): For the $k$-th dimension, $\|P\|^{(i)}_{2,d} = \|P^{(i)}\|_2$.
\end{itemize}

Before we present the proof of Theorem \ref{theorem-mano-simple}, we first propose and proof a Lemma on the lower bound on the inner product of the true gradient $g_t = \nabla f(\theta_t)$ and the normalized tangent $\hat{v}_t$. 

\begin{lemma}[]
    \label{lemma:st}
    % Given $\hat{v}_t \gets v_t \oslash \| v_t \|_{2, 0}$, $v_t \gets g_t - \hat{\theta}_t \odot \langle g_t, \hat{\theta}_t \rangle_{0}$, $\hat{\theta}_t \gets \theta_t \oslash \| \theta_t \|_{2, 0}$,
    Under the conditions of the update rule stated in Eq.~\ref{eq:mano-wo-momentum}, for $\phi_t^{(j)}$ be the angle between $g_t^{(j)} = \nabla f(\theta_t)^{(j)}$ and the parameter $\theta_t^{(j)}$, let $\sin(\phi_t^{(j)}) \geq \gamma > 0$ for tangential component $\gamma$, we have
    \begin{equation}
        \langle \nabla f(\theta_t), \hat{v}_t \rangle \geq \gamma \|\nabla f(\theta_t)\|.
    \end{equation}
\end{lemma}

\begin{proof}
    Denote the inner product as $S_t = \langle \nabla f(\theta_t), \hat{v}_t \rangle = \langle g_t, \hat{v}_t \rangle$, we have
    \begin{align}
        S_t &= \sum_{j=1}^m \langle g_t^{(j)}, \frac{v_t^{(j)}}{\|v_t^{(j)}\|} \rangle
        = \sum_{j=1}^m \langle v_t^{(j)} + \hat{\theta}_t^{(j)} \odot \langle g_t^{(j)}, \hat{\theta}_t^{(j)} \rangle, \frac{v_t^{(j)}}{\|v_t^{(j)}\|} \rangle
    \end{align}
    Because $v_t^{(j)}$ and $\hat{\theta}_t^{(j)}$ are orthogonal, the second part of the inner product is zero:
    \begin{align}
        S_t &= \sum_{j=1}^m \langle \frac{v_t^{(j)}, v_t^{(j)}}{\|v_t^{(j)}\|} \rangle = \sum_{j=1}^m \frac{\|v_t^{(j)}\|^2}{\|v_t^{(j)}\|} = \sum_{j=1}^m \|v_t^{(j)}\| % \\
        % &= \sum_{j=1}^m \| g_t^{(j)} - \hat{\theta}_t^{(j)} \odot \langle g_t^{(j)}, \hat{\theta}_t^{(j)} \rangle \|
    \end{align}
    Because $v_t$ is defined as the component of the gradient $g_t$ orthogonal to the parameter $\theta_t$, let $\phi_t^{(j)}$ be the angle between $g_t^{(j)}$ and the parameter $\theta_t^{(j)}$, we have $\|v_t^{(j)}\| = \|g_t^{(j)}\| \sin(\phi_t^{(j)})$. Thus, we can further express $S_t$ as,
    \begin{align}
        S_t &= \sum_{j=1}^m \|g_t^{(j)}\| \sin(\phi_t^{(j)})
    \end{align}
    We derive $S_t \rightarrow 0$ when the full gradient vanishes $\|g_t\|=0$ or when the gradient is perfectly parallel to the weight vector $\sin(\phi_t^{(j)})=0$. If we assume the gradient is never perfectly aligned with the weights with a tangential component, we have $\sin(\phi_t^{(j)}) \geq \gamma > 0$, we have a lower bound for $S_t$:
    \begin{align}
        S_t = \langle \nabla f(\theta_t), \hat{v}_t \rangle \geq \gamma \sum_{j=1}^{m} \|g_t^{(j)}\| = \gamma \|g_t\| = \gamma \|\nabla f(\theta_t)\|
    \end{align}
    The proof is now complete.
\end{proof}

\subsection{Deterministic Setting}

We first consider the convergence of the update rule stated in Eq.~\ref{eq:mano-wo-momentum} under a deterministic setting for the true gradient $\nabla f(\theta_t) = g_t$. Assuming the function $f$ is $L_2$-smooth ($L$-Lipschitz Continuity), i.e., for all $x, y$,
\begin{equation}
    f(y) \leq f(x) + \nabla f(x)^\top (y-x) + \frac{L}{2} \|y-x\|^2
\end{equation}
For $y = \theta_{t+1} = \theta_t - \eta \sqrt{m} \hat{v}_t$ applied to the $L$-smoothness:
\begin{align}
    f(\theta_{t+1}) &\leq f(\theta_t) + \nabla f(\theta_t)^\top (\theta_{t+1}-\theta_t) + \frac{L}{2} \|\theta_{t+1}-\theta_t\|^2 \nonumber \\
    &= f(\theta_t) + \langle \nabla f(\theta_t), (\theta_{t+1} - \theta_t) \rangle + \frac{L}{2} \|\eta \sqrt{m} \hat{v}_t\|^2 \nonumber \\
    &= f(\theta_t) - \eta \sqrt{m} \langle \nabla f(\theta_t), \hat{v}_t \rangle + \frac{L}{2} \eta^2 m^2
\end{align}
We now substitute $S_t = \langle \nabla f(\theta_t), \hat{v}_t \rangle$ to the deterministic analogue of the $L$-smoothness,
\begin{align}
    f(\theta_{t+1}) &\leq f(\theta_t) - \eta \sqrt{m} S_t + \frac{L}{2} \eta^2 m^2 \nonumber \\
    \eta \sqrt{m} S_t &\leq f(\theta_t) - f(\theta_{t+1}) + \frac{L}{2} \eta^2 m^2 \nonumber \\
    \sum_{t=0}^T \eta \sqrt{m} S_t &\leq \sum_{t=0}^T (f(\theta_t) - f(\theta_{t+1})) + \sum_{t=0}^T \frac{L}{2} \eta^2 m^2 \nonumber \\
    \sqrt{m} \eta \sum_{t=0}^T S_t &\leq f(\theta_0) - f(\theta_{T+1}) + (T+1) \frac{L}{2} m^2 \eta^2
\end{align}
Since $f(\theta_0) \geq f_{\inf}$, we have $\sqrt{m} \sum_{t=0}^T S_t \leq f(\theta_0) - f_{\inf} + (T+1) \frac{L}{2} m^2 \eta^2$. According to Lemma ~\ref{lemma:st} that we have $S_t \geq \gamma \|g_t\|$ for tangential component $\gamma$, we arrive at
\begin{align}
    \sqrt{m} \eta \sum_{t=0}^T \gamma \|\nabla f(\theta_t)\| &\leq f(\theta_0) - f_{\inf} + (T+1) \frac{L}{2} m^2 \eta^2 \nonumber \\
    \sum_{t=0}^T \|\nabla f(\theta_t)\| &\leq \frac{f(\theta_0) - f_{\inf}}{m^{\frac{1}{2}} \gamma \eta} + (T+1) \frac{L m^{\frac{3}{2}} \eta}{2\gamma} \nonumber \\
    \frac{1}{T+1} \sum_{t=0}^T \|\nabla f(\theta_t)\| &\leq \frac{1}{T+1} \frac{f(\theta_0) - f_{\inf}}{m^{\frac{1}{2}} \gamma \eta} + \frac{L m^{\frac{3}{2}} \eta}{2\gamma}
\end{align}
% For $C_1 = \frac{f(\theta_0) - f_{\inf}}{m^{\frac{1}{2}} \gamma}$, $C_2 = \frac{L m^{\frac{3}{2}}}{2\gamma}$, and 
Let $\eta \leq \frac{C}{\sqrt{T+1}}$, we have
\begin{align}
    \frac{1}{T+1} \sum_{t=0}^T \|\nabla f(\theta_t)\| &\leq \frac{1}{(T+1) \frac{C}{\sqrt{T+1}}} \left( \frac{f(\theta_0) - f_{\inf}}{m^{\frac{1}{2}} \gamma} \right) + \frac{C}{\sqrt{T+1}} \left( \frac{L m^{\frac{3}{2}}}{2\gamma} \right) \nonumber \\
    \Rightarrow \min_{t \in [0, T]} \|\nabla f(\theta_t)\| &\leq \frac{1}{\sqrt{T+1}} \left( C_1 + C_2 \right), \; C_1 = \frac{f(\theta_0) - f_{\inf}}{m^{\frac{1}{2}} \gamma C}, \; C_2 = \frac{L m^{\frac{3}{2}} C}{2\gamma}
\end{align}
Thus, we derived the speed of convergence for Eq.~\ref{eq:mano-wo-momentum} with deterministic gradient as $\min_{t \in [0, T]} \|\nabla f(\theta_t)\| \leq O(\frac{L m^{\frac{3}{2}}}{\gamma \sqrt{T}})$. We now attempt to extend this result to stochastic gradient.

\subsection{Stochastic Setting}

We now attempt to extend the above proof to the stochastic setting, assuming that $\mathbb{E}[{\xi_k}] = 0$ for gradient noise $\xi$ of sub-sampling. This commonly used assumption is equivalent to the equality that the stochastic gradient $\tilde{g}_t = \nabla f(\theta_k, \xi_k) = \nabla f(\theta_k) + \xi_k$ is an unbiased estimator of the true gradient $g_t = \nabla f(x_k)$, as demonstrated by the following derivation:
\begin{align}
    \mathbb{E}_{\xi_k}[\nabla f(x_k, \xi_k)] &= \mathbb{E}_{\xi_k}[\nabla f(x_k) + \xi_k] \nonumber \\
    \mathbb{E}_{\xi_k}[\nabla f(x_k, \xi_k)] &= \mathbb{E}_{\xi_k}[\nabla f(x_k)] + \mathbb{E}_{\xi_k}[\xi_k] \nonumber \\
    \mathbb{E}_{\xi_k}[\nabla f(x_k, \xi_k)] &= \nabla f(x_k) + 0 = \nabla f(x_k),
\end{align}
or equivalently $\mathbb{E}_{\xi_k}[\tilde{g}_k] = g_k$. 
%We also assume that the variance of the stochastic gradient is bounded by some $\sigma^2$, i.e.,
%\begin{equation}
%    \mathbb{E}_{\xi_k}[\|\nabla f(x_k, \xi_k) - \nabla f(x_k)\|^2] \leq \sigma^2,
%\end{equation}
%or quivalently $\mathbb{E}[\|\tilde{g}_k - g_k\|^2] \leq \sigma^2$. 
We now extend Lemma \ref{lemma:st} under the stochastic setting as detailed below.

\begin{lemma}[]
    \label{lemma:st-sgd}
    Under the conditions of the update rule stated in Eq.~\ref{eq:mano-wo-momentum}, assume that $\mathbb{E}_{\xi_k}[\nabla f(x_k, \xi_k)] = \nabla f(x_k)$, 
    % $\mathbb{E}[\|\tilde{g}_t - g_t\|^2] \leq \sigma^2$, 
    for $\phi_t^{(j)}$ be the angle between $\tilde{g}_t$ and the parameter $\theta_t^{(j)}$, let $\sin(\phi_t^{(j)}) \geq \gamma \geq 0$ for tangential component $\gamma$, we have
    \begin{equation}
        \mathbb{E}_{\xi_t}[\langle \nabla f(\theta_t + \xi_t), \hat{v}_t \rangle] \geq \gamma \mathbb{E}_{\xi_t}[\|\nabla f(\theta_t)\|]
    \end{equation}
\end{lemma}
\begin{proof}
    By the linearity of expectation, we have
    \begin{align}
        \mathbb{E}_{\xi_t}[\langle \nabla f(\theta_t + \xi_t), \hat{v}_t \rangle] = \langle \mathbb{E}_{\xi_t}[\nabla f(\theta_t + \xi_t)], \hat{v}_t \rangle = \langle \nabla f(\theta_t), \hat{v}_t \rangle 
    \end{align} 
    %Denote the stochastic inner product as $\tilde{S}_t = \langle \nabla f(\theta_t + \xi_t), \hat{v}_t \rangle$. 
    
    %For the RHS, recall the variance decomposition rule,
    %\begin{align}
    %    \mathbb{E}_{\xi_t}[\|\tilde{g}_t\|^2] = \|g_t\|^2 + \mathbb{E}_{\xi_t}[\|\tilde{g}_t - g_t\|^2] \leq \|g_t\|^2 + \sigma^2
    %\end{align}
    %By Jensen's Inequality that for any variable $X$, $\mathbb{E}[X] \leq \sqrt{\mathbb{E}[X^2]}$, we have
    %\begin{align}
    %    \mathbb{E}_{\xi_t}[\|\tilde{g}_t\|] \leq \sqrt{\|g_t\|^2 + \sigma^2} = \|g_t\|^2 \sqrt{1 + %\frac{\sigma^2}{\|g_t\|^2}}
    %\end{align}
    
    According to Lemma \ref{lemma:st}, we have% and the condition $\gamma \geq  \sqrt{1 + \frac{\sigma^2}{\|g_t\|^2}}$, we have 
    \begin{align}
        \langle \nabla f(\theta_t), \hat{v}_t \rangle \geq \gamma \|\nabla f(\theta_t) \| \nonumber \\
        \mathbb{E}_{\xi_t}[\langle \nabla f(\theta_t), \hat{v}_t \rangle] \geq \mathbb{E}_{\xi_t}[\gamma \|\nabla f(\theta_t)\|] \nonumber \\
        \mathbb{E}_{\xi_t}[\langle \nabla f(\theta_t), \hat{v}_t \rangle] \geq \mathbb{E}_{\xi_t}[\gamma \|\nabla f(\theta_t)\|] \nonumber \\
        \mathbb{E}_{\xi_t}[ \mathbb{E}_{\xi_t}[\langle \nabla f(\theta_t + \xi_t), \hat{v}_t \rangle] ] \geq \gamma \mathbb{E}_{\xi_t}[\|\nabla f(\theta_t)\|] \nonumber \\
        \mathbb{E}_{\xi_t}[\langle \nabla f(\theta_t + \xi_t), \hat{v}_t \rangle] \geq \gamma \mathbb{E}_{\xi_t}[\|\nabla f(\theta_t)\|]
    \end{align}
    %\begin{align}
    %    \mathbb{E}_{\xi_t}[\langle \nabla f(\theta_t + \xi_t), \hat{v}_t \rangle] = \langle \nabla f(\theta_t), \hat{v}_t \rangle \geq \gamma \| \nabla f(\theta_t) \| \geq \sqrt{1 + \frac{\sigma^2}{\|g_t\|^2}} \|g_t\|^2 \geq \mathbb{E}_{\xi_t}[\|\tilde{g}_t\|]
    %\end{align}
    The proof is now complete.
\end{proof}

\begin{proof}
    We now present the complete proof of Theorem \ref{theorem-mano-simple}, starting with the standard descent lemma for L-smoothness w.r.t. the true gradient $g_t$, and the stochastic inner product $\tilde{S}_t = \langle \nabla f(\theta_t + \xi_t), \hat{v}_t \rangle$
    \begin{align}
        \label{eq:sgd-smoothness}
        f(\theta_{t+1}) &\leq f(\theta_t) - \eta \sqrt{m} \tilde{S}_t + \frac{L}{2} \eta^2 m^2 \nonumber \\
        \mathbb{E}_{\xi_t}[f(\theta_{t+1})] &\leq \mathbb{E}_{\xi_t}[f(\theta_t)] - \eta_t \sqrt{m} \, \mathbb{E}_{\xi_t}[\tilde{S}_t] + \frac{L}{2} \eta^2 m^2 \nonumber \\
        \sum_{t=0}^T \mathbb{E}_{\xi_t}[f(\theta_{t+1})] &\leq \sum_{t=0}^T \mathbb{E}_{\xi_t}[f(\theta_t)] - \sum_{t=0}^T \eta \sqrt{m} \, \mathbb{E}_{\xi_t}[\tilde{S}_t] + \sum_{t=0}^T \frac{L}{2} \eta^2 m^2 \nonumber \\
        \eta \sqrt{m} \sum_{t=0}^T \mathbb{E}_{\xi_t}[\tilde{S}_t] &\leq f(\theta_0) - f_{\inf} + (T+1) \frac{L}{2} \eta^2 m^2
    \end{align}
    
    According to Lemma \ref{lemma:st-sgd}, we have
    \begin{align}
        \eta \sqrt{m} \sum_{t=0}^T \gamma \mathbb{E}_{\xi_t}[\|\nabla f(\theta_t)\|] &\leq f(\theta_0) - f_{\inf} + (T+1) \frac{L}{2} \eta^2 m^2 \nonumber \\
        \sum_{t=0}^T \mathbb{E}_{\xi_t}[\|\nabla f(\theta_t)\|] &\leq \frac{f(\theta_0) - f_{\inf}}{m^{\frac{1}{2}} \gamma \eta} + (T+1) \frac{L m^{\frac{3}{2}} \eta}{2\gamma} \nonumber \\
        \frac{1}{T+1} \sum_{t=0}^T \mathbb{E}_{\xi_t}[\|\nabla f(\theta_t)\|] &\leq \frac{1}{T+1} \frac{f(\theta_0) - f_{\inf}}{m^{\frac{1}{2}} \gamma \eta} + \frac{L m^{\frac{3}{2}} \eta}{2\gamma}
    \end{align}
    % For $C_1 = \frac{f(\theta_0) - f_{\inf}}{m^{\frac{1}{2}} \gamma}$, $C_2 = \frac{L m^{\frac{3}{2}}}{2\gamma}$, and 
    Let $\eta \leq \frac{C}{\sqrt{T+1}}$, we have
    \begin{align}
        \frac{1}{T+1} \sum_{t=0}^T \mathbb{E}_{\xi_t}[\|\nabla f(\theta_t)\|] &\leq \frac{1}{(T+1) \frac{C}{\sqrt{T+1}}} \left( \frac{f(\theta_0) - f_{\inf}}{m^{\frac{1}{2}} \gamma} \right) + \frac{C}{\sqrt{T+1}} \left( \frac{L m^{\frac{3}{2}}}{2\gamma} \right) \nonumber \\
        \Rightarrow \min_{t \in [0, T]} \mathbb{E}_{\xi_t}[\|\nabla f(\theta_t)\|] &\leq \frac{1}{\sqrt{T+1}} \left( C_1 + C_2 \right), \; C_1 = \frac{f(\theta_0) - f_{\inf}}{m^{\frac{1}{2}} \gamma C}, \; C_2 = \frac{L m^{\frac{3}{2}} C}{2\gamma}
    \end{align}
    The proof is now complete.
\end{proof}

%%%%%%%%%%%%%%%%%%%%%%%%%%%%%%%%%%%%%%%%%%%%%%%%%%%%%%%%%%%%%%%%%%%%%%%%%%%%%%%
%%%%%%%%%%%%%%%%%%%%%%%%%%%%%%%%%%%%%%%%%%%%%%%%%%%%%%%%%%%%%%%%%%%%%%%%%%%%%%%

\end{document}